%% file: arxiv_version.tex
\documentclass[11pt]{article}
\usepackage{amsthm,amsmath} 
\usepackage[colorlinks,linkcolor=red,anchorcolor=blue,citecolor=blue]{hyperref} 
\usepackage[capitalize]{cleveref} 
\usepackage{smile}
\usepackage{extarrows}
\usepackage[OT1]{fontenc}
\usepackage{fullpage}
\usepackage[protrusion=false,expansion=true]{microtype}
\usepackage{graphicx}
\usepackage{comment}
\usepackage{stmaryrd}
\usepackage[dvipsnames]{xcolor}

\usepackage{algorithm}
\usepackage{algorithmic}

\RequirePackage{times}
\RequirePackage{natbib}
 
\definecolor{mypink1}{rgb}{0.858, 0.188, 0.478}
\definecolor{mypink2}{RGB}{219, 48, 122}
\definecolor{mypink3}{cmyk}{0, 0.7808, 0.4429, 0.1412}
\definecolor{mygray}{gray}{0.6}

\def\mw#1{\textcolor{red}{mw:#1}}

\usepackage[textsize=tiny]{todonotes}

\def\red#1{\textcolor{red}{#1}}

\newcommand{\normmm}[1]{{\left\vert\kern-0.25ex\left\vert\kern-0.25ex\left\vert #1
		\right\vert\kern-0.25ex\right\vert\kern-0.25ex\right\vert}}

\begin{document}

\title{ \bf Sparse Feature Selection Makes Batch Reinforcement Learning More Sample Efficient}
\author
{
	Botao Hao\thanks{Deepmind. E-mail: haobotao000@gmail.com.},~
	Yaqi Duan\thanks{Princeton University. E-mail: yaqid@princeton.edu.},~
	Tor Lattimore\thanks{Deepmind. E-mail:  lattimore@google.com.}
	,~
	Csaba Szepesv\'ari\thanks{Deepmind and University of Alberta. E-mail: szepi@google.com.}
	,~
	Mengdi Wang\thanks{Princeton University. E-mail: mengdiw@princeton.edu.}
}
\date{}
\maketitle

\begin{abstract}
 This paper provides a statistical analysis of high-dimensional batch Reinforcement Learning (RL) using sparse linear function approximation. 
When there is a large number of candidate features, our result sheds light on the fact that sparsity-aware methods can make batch RL 
more sample efficient. We first consider the off-policy policy evaluation problem. To evaluate a new target policy, we analyze a 
Lasso fitted Q-evaluation method and establish a finite-sample error bound that has no polynomial dependence on the ambient dimension. 
To reduce the Lasso bias, we further propose a post model-selection estimator that applies fitted Q-evaluation to the features selected 
via group Lasso. Under an additional signal strength assumption, we derive a sharper instance-dependent error bound that depends 
on a divergence function measuring the distribution mismatch  between the data distribution and occupancy measure of the target policy. 
Further, we study the Lasso fitted Q-iteration for batch policy optimization and establish a finite-sample error bound depending 
on the ratio between the number of relevant features and restricted minimal eigenvalue of the data's covariance. In the end, we complement 
the results with minimax lower bounds for batch-data policy evaluation/optimization that nearly match our upper bounds. 
The results suggest that having well-conditioned data is crucial for sparse batch policy learning.
\end{abstract}

\section{Introduction}
We consider batch reinforcement learning (RL), where the problem is to evaluate a target policy or to learn a good policy based on a given dataset  \citep{sze10,lange2012batch,levine2020offline}.
While in online RL the central question is how to sequentially interact with the environment to balance exploration and exploitation, in batch RL the dataset is given a priori and the focus is typically on learning a near-optimal policy or evaluating a given target policy.

To handle RL systems with large or even infinite state spaces, we focus on the use of linear function approximation \citep{BeKaKo63,SchSe85,BeTs96}, that is, using a weighted linear combination of available features (aka basis functions) to represent high-dimensional transition/value functions. Results from the supervised learning literature show that the sample size needed to get accurate policy evaluations or near-optimal policies must scale at least linearly with $d$, the number of features \citep[e.g., Example 15.14 of ][]{Wa19}.

We leverage the idea of sparse approximation and focus on situations when a smaller number of 
\emph{unknown}  relevant
features is sufficient for solving the RL problem. 
Sparse regression has proved to be a powerful method for high-dimensional statistical learning with limited data \citep{tibshirani1996regression,chen2001atomic, Buneaetal07,bickel2009simultaneous,RiGr14}, and we will borrow techniques from the sparse learning literature to improve the sample efficiency of batch RL, an idea with a considerable history in RL as witnessed by our literature review that follows below.

\subsection{Contributions}
First, we consider the problem of \emph{off-policy policy evaluation} (OPE), where the objective is to estimate the value function of a \emph{target policy} from batch episodes that were generated from a possibly different behavior policy. To promote sparse solutions, we use fitted value iteration that iteratively fits state-action value functions using linear regression with $\ell_1$-regularization. We call this procedure as \emph{Lasso fitted  Q-evaluation} (Algorithm \ref{alg:batch_eva}). 
For the analysis of our methods we introduce the novel notion of \emph{sparse linear MDPs}
and then establish a finite-sample error bound for Lasso fitted Q-evaluation in sparse linear MDPs
that depends linearly on the number of relevant features and the restricted minimal eigenvalue.
Importantly, the bound has no polynomial dependence on $d$, as long as the true MDP model is sparse. 
This appears to be the first theoretical bound for sparse batch policy evaluation. 

Second, to reduce the Lasso bias, we propose an improved \emph{post model-selection estimator} (Algorithm \ref{alg:evaluation_restricted}) that applies fitted Q-evaluation with a smaller feature set that is selected using group Lasso. Under an additional separability assumption, we derive a sharper and nearly minimax-optimal error bound that is instance-dependent.
The error bound is determined by a divergence function measuring the distribution mismatch, \emph{restricted over the reduced feature space}, between the data distribution and the occupancy distribution of the target policy. This divergence defined over the reduced feature space is significantly smaller than its counterpart over the full $d$-dimensional space. In other words, {\it sparse feature selection reduces the distribution mismatch}.
We also provide a nearly-matching lower bound, and these two results sharply characterize the statistical limits of sparse off-policy evaluation.
 
We extend our analysis to the batch policy optimization problem, again in sparse linear MDPs.
We analyze the \emph{Lasso fitted Q-iteration} (Algorithm \ref{alg:batch_opt}) and show that the $\ell_{\infty}$-norm of policy error depends linearly on the ratio between the number of relevant features and the restricted minimal eigenvalue of the dataset's covariance matrix. Finally, we establish a minimax lower bound for sparse batch policy learning and show that the lower bound also depends on the aforementioned ratio. This is the first lower bound result, to the authors' best knowledge, demonstrating the critical role played by the minimal eigenvalue of the dataset's covariance matrix and the construction is highly non-trivial. The upper and lower bounds validate the belief that well-conditioned data is crucial for sample-efficient policy learning.

\subsection{Related work}
\paragraph{Off-policy policy evaluation (OPE).} OPE often serves the starting point of batch RL.
A direct approach was to fit value function from data using approximate dynamic programming, e.g., the policy evaluation analog of
fitted Q-iteration \citep{ernst2005tree, munos2008finite, le2019batch} or least square policy iteration \citep{lagoudakis2003least}. Another popular class of OPE methods used importance sampling to get unbiased value estimate of a new policy \citep{precup2000eligibility} and improved by doubly-robust technique to reduce the variance \citep{jiang2016doubly, thomas2016data}. To  alleviate the curse of horizon \citep{li2015toward, jiang2016doubly,yin2020asymptotically}, marginalized importance sampling was suggested by estimating state marginal importance ratio without reweighting the entire trajectory \citep{hallak2017consistent, liu2018breaking, xie2019towards}. 
In general, estimating marginalized importance ratio could be sample-expensive and even intractable. Recently, practical duality-inspired methods were developed for estimating this ratio using function approximation  \citep{nachum2019dualdice, uehara2019minimax, zhang2020gendice, zhang2020gradientdice,yang2020off}.

On the theoretical side, \cite{uehara2019minimax, yin2020asymptotically, kallus2020double} established asymptotic optimality and efficiency for OPE in the tabular setting.  \cite{duan2020minimax} showed that fitted Q-evaluation with linear function approximation is minimax optimal 
and provided matching upper and lower bounds that depend on a distribution mismatch term. Another closely related work was by \cite{le2019batch} who studied batch policy evaluation and optimization with more general function approximation. They showed the complexity of batch RL depends on the complexity of the function class, assuming a ``concentration coefficient'' condition \citep{munos2008finite} that the state-action visitation density is bounded entrywisely across policies. More recently, \cite{uehara2019minimax} provided theoretical investigations into OPE using general function approximators for marginalized importance weights and value functions but did not show the statistical optimality.

\paragraph{Sparse learning in RL.} 
The use of feature selection by regularization in RL has been explored in a number of prior works. \cite{kolter2009regularization, geist2011, hoffman2011regularized, painter2012greedy} studied \emph{on-policy} evaluation with $\ell_1$-regularization for temporal-difference (TD) learning but none of them come with a theoretical analysis. \cite{liu2012regularized} studied off-policy evaluation by regularized TD learning but only provided algorithmic convergence guarantee without statistical error analysis.

\cite{ghavamzadeh2011finite, geist2012dantzig} proposed Lasso-TD with finite-sample statistical analysis for estimating the value function in \emph{Markov reward process}. In particular, they derived in-sample \emph{prediction error} bound $\cO((s\log (d)/\psi n)^{1/2})$ under $\psi$-minimum eigenvalue condition on the empirical feature gram matrix. Although this bound also has no polynomial dependency on $d$, in-sample prediction error generally can not be translated to the estimation error of target policy in the OPE problem and their bound can not characterize the distribution mismatch between behavior policy and target policy. On the other hand, no minimax lower bound has been investigated so far.

In addition, \cite{massoud2008regularized, farahmand2016:jmlr} considered $\ell_2$-regularization in fitted Q-iteration/policy iteration for policy optimization in a reproducing kernel Hilbert space, and finite-sample performance bounds for
these algorithms were proved built on a coefficient concentration condition. 
\cite{calandriello2014sparse} developed Lasso fitted Q-iteration for sparse multi-task reinforcement learning and assumed  a generative model \citep{kakade2003sample} for sampling transitions.
\citet{ibrahimi2012efficient} derived a $\cO(p\sqrt{T})$ regret bound in high-dimensional sparse linear quadratic systems where $p$ is the dimension of the state space.

\paragraph{Sparse linear regression.} Sparse regression receives considerable attention in high-dimensional statistics in the past decade. Lasso \citep{tibshirani1996regression},  is arguably the most widely used method to conduct sparse regression. Theoretical analysis of Lasso is well-studied in \cite{zhao2006model, bickel2009simultaneous, wainwright2009sharp}. For a thorough review of Lasso as well as high-dimensional statistics, we refer the readers to \cite{hastie2015statistical, Wa19,buhlmann2011statistics}. However, extending existing analysis from regression to batch RL is much more involved due to the complex optimization structure, non-i.i.d data collection, and covariate shift.

\section{Preliminaries}
\textbf{Notations.}
Denote $[n]=\{1,2,\ldots, n\}$. We identify vectors and functions in the obvious way by fixing an (arbitrary) ordering over the domain of the function. 
Similarly, for $\cS$ finite and $A\in \RR^{\cS \times \cS}$ will be treated as an $|\cS|\times |\cS|$ matrix.
For a vector $x\in\mathbb R^d$, a matrix $X\in\mathbb R^{d\times d}$ and an index set $\cS\subseteq [d]$, 
we define $x_{\cS} = (x_j)_{j\in S}\in\mathbb R^{\cS}$ 
be the subvector over the restricted index set $\cS$
and $X_{\cS\times \cS} = (X_{ij})_{i,j\in \cS}\in\mathbb R^{\cS \times \cS}$. We write $X_{j\cdot}$ as the $j$th row of $X$.
Denote $\lambda_{\min}(X),\lambda_{\max}(X)$ the smallest and largest eigenvalues of a symmetric matrix $X$. 
Denote $\lesssim,\gtrsim$ as ``approximately less/greater than", omitting non-leading order terms (constant and polylog factors).
For a set $\cS$ let $\Delta_{\cS}$ denote the set of probability distributions over $\cS$.

\subsection{Problem definition}  
A finite, infinite-horizon discounted Markov decision process (DMDP) can be described by the tuple $M=(\cX,\cA,P,r,\gamma)$. 
Here, $\cX$ is a finite set of states, 
$\cA$ is a finite set of actions, 
$P:\cX\times \cA \to \Delta_{\cX}$ is the transition probability function, 
$r:\cX\times \cA\to [0, 1]$ is the reward function and 
$\gamma\in (0, 1)$ is the so-called discount factor. 
In this paper, for the sake of simplicity, we stick to finite DMDPs. 
However, our results can be extended to more general cases with routine work.

We define a (stationary) policy $\pi$ as a
$\cX\to \Delta_{\cA}$ map, mapping states to distributions over actions. 
A policy, a distribution $\xi_0$ over $\cX$
and a DMDP $M$ together give rise to a probability measure $\mathbb P^{\pi}$ over the set of infinitely long state-action histories: for a history of form $x_1,a_1,r_1,x_2,a_2,r_2,\dots$,  
$\mathbb P^{\pi}(x_1) = \xi_0(x_1)$ and for $t=1,2,\dots$,
$\mathbb P^{\pi}(a_t|x_1,a_1,\dots,x_{t-1},a_{t-1},x_t) = \pi(a_t|x_t)$
 and 
$\mathbb P^{\pi}(x_{t+1}|x_1,a_1,\dots,x_{t-1},a_{t-1},x_t,a_t) = P(x_{t+1}|x_t,a_t)$. We denote by $\mathbb E^{\pi}$ the corresponding expectation operator. 
The \emph{value} of policy $\pi$ given the initial state distribution $\xi_0$ is 
$v^{\pi}_{\xi_0} := \mathbb E^{\pi}[\sum_{t=0}^\infty \gamma^t r(x_t, a_t) ]$, where $\mathbb E^{\pi}$ depends on $\xi_0$ but just this dependence is not shown.
We define $v^\pi:\cX \to \mathbb R$, the \emph{value function} of policy $\pi$, by letting $v^\pi(x)$ denote the value of $\pi$ when it is started from state $x$. 

A nonstationary policy is a sequence of maps from histories to probability distributions over actions and such a policy similarly induces a probability measure $\mathbb P^{\pi}$ and an underlying  expectation operator $\mathbb E^{\pi}$.
An \emph{optimal policy} $\pi^*$ maximizes the value from every state: $v^{\pi^*}(x) = \max_\pi v^{\pi}(x)$ for any $x\in \cX$ where the maximum is over all policies, including nonstationary ones. As is well known, for finite DMDPs,  a stationary optimal policy always exist \citep[e.g.,][]{sze10}.
The value function shared by optimal policies is called the \emph{optimal value function} and is denoted by $v^*$.

We consider the following learning and optimization problems.
The learner knows the state space $\cX$ and action space $\cA$. 
The reward function $r$ is given in the form of a black box, which the learner can use to evaluate $r(x,a)$ for any pair of $(x,a)\in \cX \times \cA$.
The only unknown is the transition probability function $P$.
The learner is given 
a random dataset $\cD = \{(x_n, a_n, x_n')\}_{n=1}^N$ generated by using a (possibly nonstationary and unknown)
\emph{behavior policy} $\bar{\pi}$ in the DMDP $M$ starting from some initial distribution which may be different from $\xi_0$.
We study two fundamental batch RL tasks: 
\begin{itemize}
    \item \emph{Off-policy policy evaluation:}
    given $\cD$ and black box access to a \emph{target policy} $\pi$, $\xi_0$ and $r$,
    estimate the value, $v^{\pi}_{\xi_0}$, of $\pi$;
    \item   \emph{Batch policy optimization:}
    given $\cD$ and black box access to $r$, find an optimal policy. 
\end{itemize}

\paragraph{Bellman operators.} 
By slightly abusing the notation, we will view the transition probability kernel $P$ as a left linear operator, mapping from $\RR^\cX$ to $\RR^{\cX \times \cA}$: 
\begin{equation*}
(P v)(x,a):= \sum_{x'} P(x'|x,a) v(x').
\end{equation*} 
The \emph{Bellman optimality operator} $\cT:\RR^{\cX}\to \RR^{\cX}$ is defined as
$$
[\cT v](x) = \max_{a\in \cA}(r(x,a) + \gamma P v(x,a)), \ \forall x\in \cX.
$$
The function $v^*\in \RR^{\cX}$  is the unique solution to the Bellman optimality equation $v= \cT v$. 
 The state-action value function of a policy $\pi$ is defined as 
$$
Q^{\pi} (x,a) = r(x,a) + \gamma P v^{\pi}(x,a) .
$$
 We also introduce the Bellman operator $\cT_{\pi}:\RR^{\cX}\to \RR^{\cX}$ 
  for policy $\pi$
 as:  
 $$
 [\cT_{\pi}v] (x) = \sum_{a\in \cA} \pi(a|x) \left( r(x,a) + \gamma  P v (x,a)\right)\,, \ \forall x\in \cX.
 $$
The function $v^{\pi}$ is the unique solution to the Bellman equation $v^{\pi}=\cT_{\pi}v^{\pi}$.

\subsection{Sparse linear Markov decision process}
Bellman's curse of dimensionality refers to that DMDPs tend to have a state space where individual states have many components that can take on many values, hence the state space is intractably large. An early idea due to \citet{BeKaKo63,SchSe85} is to use linear combination of finitely many, say, $d$ basis functions, to express value functions and thus allow for computation and storage that is independent of the size of the state-space.
Following this idea, let $\phi: \cX \times \cA \to \RR^d$ be a feature map which assigns to each state-action pair a $d$-dimensional feature vector. A feature map combined with a parameter vector $w\in \RR^d$ gives rise to the linear function $g_w:\cX\times \cA \to \RR$ defined by $g_w(x,a) = \phi(x,a)^\top w$, $(x,a)\in \cX \times \cA$,
and the subspace  $\cG_\phi = \{ g_w \,:\, w\in \RR^d \}\subseteq \RR^{\cX \times \cA}$. 

Fix a feature map $\phi$.  For $f:\cX \times \cA \to \RR$, let $f^\pi:\cX \to \RR$ be defined by $f^\pi(x) = \sum_a \pi(a|x)f(x,a)$.
The idea of \citet{BeKaKo63,SchSe85} to use features to find optimal policies with computation whose cost only depends on $d$ and not on the cardinality of $\cX$, boosted with the use of action-value functions, 
can be formally justified 
in DMDPs such that $\cG_\phi$ is closed under the application of any of the 
operators $\tilde{\cT}_\pi: \cX \times \cA \to \cX \times \cA$, $\tilde{\cT}_\pi f \mapsto r+ \gamma P f^\pi$ for any policy $\pi$ of the DMDP. 
This leads to the definition of \emph{linear DMDPs}\footnote{A different definition is called linear kernel MDP that the MDP transition kernel can be parameterized by a small number of parameters \citep{yang2019reinforcement, cai2019provably, zanette2020frequentist, zhou2020provably}.}  \citep{yang2019sample, jin2019provably,agarwal2020pc}. Slightly deviating from the original definitions, here we say 
that {\it a DMDP is linear if for some feature map $\phi$,
$\cG_\phi$ is closed under the operators 
$f \mapsto P f^\pi$ for any policy $\pi$ of the DMDP. } The following is an immediate corollary of the definitions and forms the basis of our methods:
\begin{proposition}\label{prop:linmdp}
Let $M$ be linear under $\phi$ and $\pi$ a policy for $M$.
Then, there exists some $w\in \RR^d$ such that $g_w =P v_\pi$ and
$v_\pi = r^\pi+\gamma g_w^\pi$.
Further, $w$ also satisfies the linear equation
$g_w = P r^\pi + \gamma P g_w^\pi$.
Conversely, if $w$ satisfies this last equation then $v_\pi = r^\pi+\gamma g_w^\pi$.
\end{proposition}

When little a priori information is available on how to choose the features, agnostic choices often lead to dimensions which can be as large as the number of samples $n$, if not larger. Without further assumptions, no procedure can achieve nontrivial performance guaranteed even when just considering simple prediction problems (e.g., predicting immediate rewards). 
However, effective learning with many more features than the sample-size is possible when only $s\ll d$ features are relevant.
This motivates our assumption of sparse linear DMDPs.

\begin{assumption}[{\bf Sparse linear DMDPs}]	\label{assum:sparse_MDP}
Fix a feature map  $\phi:\cX \times \cA \to \RR^d$.
We assume the unknown DMDP $M$ is $(s,\phi)$-sparse such that there exists $\cK\subseteq [d]$ with $|\cK|\le s$ and some functions $\psi(\cdot) = (\psi_k(\cdot))_{k\in\cK}$ such that
$P(x' |x,a) = \sum_{k\in \cK} \phi_k(x,a) \psi_k(x'),  \forall~x,a.$

\end{assumption}
We denote by $\cM_{\phi,s}(\cX,\cA,\gamma)$ the set of all $(s,\phi)$-sparse DMDP instances. Our second assumption concerns the dataset: 
\begin{assumption}\label{assm:data_colle} 
The dataset $\cD$ consists of $N=KL$ samples from $K$ independent episodes $\tau_1,\ldots, \tau_K$. Each episode $\tau_k$ has $L$ consecutive transitions generated by some unknown behavior policy $\bar{\pi}_k$ giving rise to a sample path $\tau_k = (x_0^{(k)}, a_0^{(k)}, x_0^{(k)'}, \ldots, x_{L-1}^{(k)}, a_{L-1}^{(k)}, x_{L-1}^{(k)'})$.
\end{assumption}

\section{Sparsity-Aware Off-Policy Policy Evaluation}

In this section we consider the off-policy policy evaluation (OPE) problem, i.e., to estimate the value of a target policy $\pi$ from logged experiences $\mathcal{D}$ generated using unknown behavior policies.
We propose two sparsity-aware algorithms to approximate state-action functions using sparse parameters.

\subsection{Lasso fitted Q-evaluation}

We propose a straightforward modification of the fitted Q-evaluation method
\citep{le2019batch, duan2020minimax} to account for sparsity.
For $t=1,2,\dots,T-1$, the algorithm
produces $\hat w_{t+1}$ using Lasso-regularized regression
so that $g_{\hat w_{t+1}} \approx P( (r+\gamma g_{\hat w_t})^\pi)$.
To make the errors of different steps independent, 
the dataset $\cD$  is split into $T$ nonoverlapping folds $\cD_1,\dots,\cD_T$,
so that all folds have $R=K/T$ episodes in them
and only data from fold $t$ is used in step $t$.
To define the algorithm, it is useful to introduce 
$Q_w = r+\gamma g_w$ where $w\in \RR^d$. For $a<b$, we also define the
 operator $\Pi_{[a,b]}: \RR \to [a,b]$ that projects its input to $[a,b]$, i.e., $\Pi_{[a,b]}(x) = \max(\min(x,b),a)$.
The pseudocode is given as Algorithm \ref{alg:batch_eva}. In the last step, $m=N$ samples are used to produce the final output to guarantee that the error introduced by the Monte-Carlo averaging is negligible compared to the rest.
\begin{remark}
The last step Monte Carlo averaging is only for numerical integration, where the samples are newly drawn inside the algorithm (independent from batch data), so there is no bias here. We set $m=N$ to simplify the theory but it could be much larger than $N$ for a more accurate approximation.
\end{remark}

{\small
\begin{algorithm}[htb!]
	\caption{Lasso fitted Q-evaluation}
	\begin{algorithmic}[1]\label{alg:batch_eva}
		\STATE
		\textbf{Input:} Initial distribution $\xi_0$, target policy $\pi$, 
		$T$ folds of dataset $\{\cD_t\}_{t=1}^T$ of overall size $N$, 
		regularization parameter $\lambda_1>0$, $m:=N$, $\hat{w}_0= 0\in \RR^d$.
		\FOR{$t=1,2,\ldots, T$}
		\STATE Calculate regression targets for $(x_i, a_i, x_i')\in\cD_t$:
		$$
		y_i = \sum_{a} \pi(a|x_i') Q_{\hat{w}_{t-1}}(x_i', a)\,. $$
		
		\STATE Fit $\hat{w}_{t}$ through sparse linear regression
		\begin{equation*}
		\begin{split}
		 &\hat{w}_{t}= \argmin_{w}\Big\{
		  \frac{1}{|\cD_t|}\sum_{(x_i,a_i,x_i')\in\cD_t}(\Pi_{[0,1/(1-\gamma)]}y_i-\phi(x_i, a_i)^{\top}w)^2+\lambda_1 \|w\|_1\Big\}\,.   
		\end{split}
		\end{equation*}
		\ENDFOR
			\STATE
		\textbf{Output:} $\hat{v}_{\hat{w}_T}^{\pi}= \frac1m\sum_{u=1}^m  \Pi_{[0,1/(1-\gamma)]} (Q_{\hat{w}_T}(\tilde{x}_u,\tilde{a}_u))$
		with $\tilde{x}_u\sim \xi_0$, $\tilde{a}_u\sim \pi(\cdot|\tilde{x}_u)$.
	\end{algorithmic}
\end{algorithm}
}

\subsection{Post model-selection fitted Q-evaluation} 

Sparse regularization is known to induce a small bias in regression. However, this bias could get compounded through the iterative procedure of Algorithm \ref{alg:batch_eva}. 
To avoid such bias and improve the accuracy, we aim to identify the set of relevant features $\cK$ before evaluating the policy based on the following proposition.  
\begin{proposition}\label{prop:linmdpf}
Under Assumption \ref{assum:sparse_MDP}, there exists a $d\times d$ matrix $K^{\pi}$ such that $\forall~(x,a)\in \cX\times \cA$
\begin{equation}\label{eqn:identity}
\begin{split}
    \mathbb E_{x'\sim P(\cdot|x,a)}[\phi^{\pi}(x')^{\top}\mid x,a] = \phi(x,a)^{\top} K^{\pi},
\end{split}
\end{equation}
where $\phi^{\pi}(x) = \sum_a \pi(a|x)\phi(x,a) $ and  all but $s$ rows of $K^{\pi}$ are identically zero. This $K^{\pi}\in\mathbb R^{d\times d}$ is known as the \emph{matrix mean embedding} \citep{duan2020minimax}.
\end{proposition}

Thus we propose to estimate the set of relevant features $\cK$ using group lasso \citep{yuan2006model}.
Once the relevant feature set is identified, any regular policy evaluation method can be used
over the learned feature set $\hat{\cK}$.
In Algorithm \ref{alg:evaluation_restricted},
for the ease of comparability with the previous method, we consider vanilla fitted Q-evaluation.  

\begin{remark}
One may wonder whether it is necessary  to refit the iterative regression and why not simply use the estimated $\hat K^{\pi}$ to get a plug-in estimator. This is because refitting typically performs strictly better than direct regularized learning and has less bias, as long as the feature selection succeeds \citep{belloni2013least}.
\end{remark}

{\small
\begin{algorithm}[htb!]
	\caption{Post model-selection fitted Q-evaluation}
	\begin{algorithmic}[1]\label{alg:evaluation_restricted}
		\STATE
		\textbf{Input:} initial distribution $\xi_0$, target policy $\pi$, dataset $\cD$ of size $N$, $m:=N$, 
		number of iterations $T$,  $\lambda_2,\lambda_3>0$.
		\STATE   Estimate $\hat{K}^{\pi}$ through: 
\begin{equation}\label{eqn:group_lasso}
\begin{split}
     &\hat{K}^{\pi} =\argmin_{K\in\mathbb R^{d\times d} }\Big\{
      \frac{1}{Nd}\sum_{n=1}^N\big\|\phi^{\pi}(x_n')^{\top}-\phi(x_n, a_n)^{\top} K\big\|_2^2 + \lambda_2\sum_{j=1}^d \|K_{j\cdot}\|_2\Big\}.
\end{split}
\end{equation}
\STATE Find $\hat{\cK} = \{j\in[d]: \|\hat{K}^{\pi}_{j\cdot}\|_2 \neq 0\}$,  the estimated set of relevant features. Here, $\hat{K}^{\pi}_{j\cdot}$ refers to the $j$th row of $\hat{K}^{\pi}$.

		\STATE
		\textbf{Initialize:} $\hat{w}_0= 0\in \RR^{|\hat{\cK}|}$.
		\FOR{$t=1,2,\ldots, T$}
		\STATE Calculate regression targets for $n\in[N]$: 
		$$
		y_n = r(x_n,a_n) + \gamma \sum_{a}\pi(a|x_n')  [\phi(x_n',a)_{\hat{\cK}}]^{\top}\hat{w}_{t-1} \,.
		$$ 
		\STATE Update $\hat{w}_{t}$ through 
		\begin{equation*}
		\begin{split}
		    \hat{w}_{t+1}= \argmin_{w\in\mathbb R^{|\hat{\cK}|}}\Big\{\frac{1}{N}\sum_{n=1}^{N}(y_n -[\phi(x_n, a_n)_{\hat{\cK}}]^{\top}w)^2 + \lambda_3 \|w\|_2^2\Big\}\,.
		\end{split}
		\end{equation*}
		\ENDFOR
			\STATE
		\textbf{Output:} $\hat{v}_{\hat{w}_T}^{\pi}= \frac1m\sum_{u=1}^m  \Pi_{[0,1/(1-\gamma)]} (Q_{\hat{w}_T}(\tilde{x}_u,\tilde{a}_u))$
		with $\tilde{x}_u\sim \xi_0$, $\tilde{a}_u\sim \pi(\cdot|\tilde{x}_u)$.
	\end{algorithmic}
\end{algorithm}
}

\section{Performance Bounds For Sparse Off-Policy Evaluation}\label{sec:bound_OPE}
We study the finite-sample estimation error of Algorithms \ref{alg:batch_eva}, \ref{alg:evaluation_restricted}. 
All the technical proofs are deferred to the Appendix. Let $\Sigma$ be the expected uncentered covariance matrix of the batch data, given by
\begin{equation}\label{eqn:expected_cov}
    \Sigma:=\mathbb E\left[\frac{1}{L}\sum_{h=0}^{L-1}\phi(x_{h}^{(1)},a_{h}^{(1)})\phi(x_{h}^{(1)},a_{h}^{(1)})^{\top}\right] \,,
\end{equation} 
where $L$ is the length of one episode. We need a notion of restricted eigenvalue that is common in high-dimensional statistics \citep{bickel2009simultaneous, buhlmann2011statistics}. 
\begin{definition}[{\bf Restricted eigenvalue}]\label{def:RE}
Given a positive semi-definite matrix $Z\in\mathbb R^{d\times d}$ and integer $s\geq 1$, define the restricted minimum eigenvalue of $Z$ as $C_{\min}(Z, s):=$
\begin{equation*}
\min_{\cS\subset [d], |\cS|\leq s}\min_{\bbeta\in\mathbb R^d}\left\{\frac{\langle \bbeta, Z\bbeta \rangle}{\|\bbeta_{\cS}\|_2^2}:  \|\bbeta_{\cS^c}\|_{1}\leq 3\|\bbeta_{\cS}\|_{1}\right\} \,.
\end{equation*}
\end{definition}

\begin{remark}
The restricted eigenvalue $C_{\min}(\Sigma, s)$ characterizes the quality of the distribution that generates the batch data set $\cD$. We need $C_{\min}(\Sigma, s)>0$ meaning that the data is well-conditioned or the behavior policy provides good coverage over relevant features. This is a key condition to guarantee the success of sparse feature selection \citep{bickel2009simultaneous}. To ensure the success of policy evaluation/optimization with linear function approximation, similar assumptions regarding $\Sigma$ in RL literature also appear in \cite{abbasi2019politex} (Assumption A.4), \cite{duan2020minimax} (Theorem 2), \cite{lazic2020maximum} (Assumption A.3), \cite{abbasi2019exploration} (Assumption A.3)  and \cite{agarwal2020theory} (Assumption 6.2).  
\end{remark}
\subsection{Finite-sample error bounds of Algorithm \ref{alg:batch_eva}}\label{sec:error_bound_OPE}

We first provide a statistical error bound for the Lasso fitted Q-evaluation (Algorithm \ref{alg:batch_eva}). 
\begin{theorem}[\textbf{OPE error bound for Algorithm \ref{alg:batch_eva}}]\label{thm:upper_bound_eva1}
 Suppose Assumptions \ref{assum:sparse_MDP}, \ref{assm:data_colle} hold and $C_{\min}(\Sigma, s)>0$. Let
	Algorithm~\ref{alg:batch_eva} take $N$ samples satisfying $N\gtrsim s^2\log(d/\delta)L(1-\gamma)^{-1}/C_{\min}(\Sigma, s).$ 
 Set the number of iterations $T = \Theta(\log (N/(1-\gamma))/(1-\gamma))$ and $\lambda_1=(1-\gamma)^{-1}\sqrt{T\log (2d/\delta)/N}$. Then, with probability at least $1-2\delta$,
\begin{equation}\label{eqn:worse_case_bound}
      \big|\hat{v}^{\pi}_{\hat{w}_T}-v^{\pi}\big| \lesssim  \frac{1}{C_{\min}(\Sigma, s)}\sqrt{\frac{s^2\log(d/\delta)}{N(1-\gamma)^{5}}}.
\end{equation}
\end{theorem}

Theorem \ref{thm:upper_bound_eva1} shows that the OPE error for sparse linear DMDPs depends linearly on $s/C_{\min}(\Sigma,s)$. For comparison, when the linear DMDPs model is not sparse, \cite{duan2020minimax} proved the error bound (using our notations) of the form
$$
|\hat{v}^{\pi}-v^{\pi}| \lesssim \sqrt{d/(C_{\min}(\Sigma, d)N(1-\gamma)^4)}.
$$ 
From the Definition \ref{def:RE}, $C_{\min}(\Sigma, d)<C_{\min}(\Sigma, s)$. Comparing the two results (setting $\delta=1/N$),
we expect the new error bound to be significantly tighter
, i.e., $C_{\min}(\Sigma,d)\frac{s^2 \log(dN)}{1-\gamma} \ll C^2_{\min}(\Sigma,s) d$, 
when there is a high level of sparsity ($s \ll d$).

\subsection{Finite-sample error bounds of Algorithm \ref{alg:evaluation_restricted}}

Next, we give a result for the post-selection model estimator for OPE (Algorithm \ref{alg:evaluation_restricted}).  
We will show that this algorithm provides a more accurate estimate under the additional condition that every relevant feature plays a nontrivial role in the transition dynamics. 
\begin{assumption}[{\bf Minimal signal strength}] \label{ass:signal}
For some given $\delta>0$,
 the \emph{minimum signal strength} satisfies
$$
    \min_{j\in\cK}  \|K^{\pi}_{j\cdot}\|_2/\sqrt{d} \geq \frac{64\sqrt{2}s}{C_{\min}(\Sigma,s)} \sqrt{\frac{2\log(2d^2/\delta)}{N}},
$$
where $K_{j\cdot}^{\pi}$ is the $j$th row of $K^{\pi}$ defined in Eq.~\eqref{eqn:identity}.
\end{assumption}

Then we provide a critical lemma showing that the group lasso step in Algorithm \ref{alg:evaluation_restricted} is guaranteed to identify a sufficiently sparse feature set including all the relevant features with high probability.
\begin{lemma}[{\bf Feature screening}]\label{lemma:group_lasso}
 Suppose Assumptions \ref{assum:sparse_MDP}, \ref{assm:data_colle}, \ref{ass:signal} hold and $C_{\min}(\Sigma, s)>0$. Set the regularization parameter 
    $ \lambda_2 =4\sqrt{2\log(2d^2/\delta)/(Nd)}$ for some $\delta>0$ and let the sample size satisfy $N\gtrsim Ls^2\log (d/\delta)/C_{\min}^2(\Sigma, s)$. Then with probability at least $1-\delta$, the size of learned relevant feature set $\hat{\cK}$ satisfies $|\hat{\cK}|\lesssim s$ and $\hat{\cK}\supseteq \cK$ where $\cK$ is the true relevant feature set of $M$.
\end{lemma}

Now we analyze the policy evaluation error of Algorithm~\ref{alg:evaluation_restricted}. According to Cramer-Rao lower bound for tabular OPE \citep{jiang2016doubly} 
and the minimax lower bound for OPE with linear function approximation \citep{duan2020minimax}, 
we expect the optimal OPE error to depend on the \emph{distribution mismatch} between the target policy and the behavior policy that generated the data. To define the notion of distribution mismatch, we first need the notion of occupancy measures:

\begin{definition}[{\bf Occupancy measures}]\label{def:om}
Let $\bar{\mu}$
be the expected occupancy measure of observations $\{(x_n, a_n)\}_{n=1}^N$: $\bar{\mu}(x,a) = \sum_{n=1}^N\mathbb P(x_n=x, a_n = a)/N$ and $\mu^{\pi}$ be the discounted occupancy distribution of $(x_h, a_h)$ under policy $\pi$ and initial distribution $\xi_0$: 
$\mu^{\pi}(x,a) = (1-\gamma)\mathbb E^{\pi}[\sum_{h=0}^{\infty}\gamma^h\ind(x_h=x, a_h= a)], \forall~x,a.$
\end{definition}
Inspired by Theorem~5 of
\citet{duan2020minimax},
we will measure the distribution mismatch using \emph{restricted chi-square divergences} between $\bar{\mu}$ and $\mu^\pi$.
\begin{definition}[{\bf Restricted chi-square divergence}]\label{def:chi-square}
Let $\cG$ be a set of real-valued functions over $\cX$ and let $p_1$ and $p_2$ be probability distributions over $\cX$.
We define the \emph{$\cG$-restricted chi-square divergence} (or $\chi^2_{\cG}$-divergence) between $p_1$ and $p_2$ as
$$
    \chi^2_{\cG}(p_1, p_2) := \sup_{f\in\cG}\frac{\mathbb E_{p_1}[f(x)]^2}{\mathbb E_{p_2}[f(x)^2]}-1.
$$
\end{definition}

By using the feature screening Lemma \ref{lemma:group_lasso}, and a similar analysis as by \cite{duan2020minimax}, we obtain the following 
instance-dependent error bound for sparse off-policy evaluation.  
\begin{theorem}[{\bf Instance-dependent error bound for sparse OPE}]
\label{thm:upper_bound_eva2}
 Suppose Assumptions \ref{assum:sparse_MDP}, \ref{assm:data_colle}, \ref{ass:signal} hold and $C_{\min}(\Sigma, s)>0$. Let $\delta\in(0, 1)$ and assume that
	Algorithm~\ref{alg:evaluation_restricted} is fed with $N$ samples satisfying 
$N\gtrsim L\log (d^2/\delta)s^2/C_{\min}^2(\Sigma, s) +  \gamma^2L\log(s/\delta)s/(1-\gamma)^2.$ 
 Set  $\lambda_2=4\sqrt{2\log(2d^2/\delta)/Nd}, \lambda_3 = \lambda_{\min}(\Sigma)\log(12|\hat{\cK}|/\delta)L|\hat{\cK}|$.
Letting the number of iterations $T\to\infty$, the following holds with probability at least $1-3\delta$,
\begin{equation}\label{eqn:instance_dependent_bound}
\begin{split}
   |\hat{v}^{\pi}_{\hat{w}}-v^{\pi}| \lesssim \sqrt{1+\chi^2_{\cG(\hat{\cK})}(\mu^{\pi}, \bar{\mu})}\sqrt{\frac{\log (1/\delta)}{N(1-\gamma)^4}},
\end{split}
\end{equation}
where $\bar\mu$ is the data generating distribution, $\cG(\hat{\cK})$ is the reduced feature space.
\end{theorem}
\begin{remark}[{\bf Reduced distribution mismatch via sparse feature selection.}]

The OPE error bound of Theorem \ref{thm:upper_bound_eva2} depends on the statistics $\chi^2_{\cG(\hat{\cK})}(\mu^{\pi},\bar \mu)$ that quantifies the distribution mismatch between data and the target policy. 
This result implies the uncertainty for evaluating a new policy from batch data crucially and jointly depends on the two distributions as well as the function class used for fitting.
When $\hat\cK$ is a small subset of $[d]$, we have $\chi^2_{\cG(\hat{\cK})} \ll \chi^2_{\cG([d])}$. Therefore our instance-dependent error bound is expected to be significantly smaller than its counterpart that does not exploit sparsity.
\end{remark}

\subsection{Minimax lower bound for OPE}
To complete the picture, we provide a minimax lower bound of off-policy evaluation for the class of sparse linear DMDPs $\cM_{\phi,s}(\cX,\cA,\gamma)$ (Assumption \ref{assum:sparse_MDP}). The proof is an adaptation of the respective lower bound proof for linear MDPs (Theorem 3 in \cite{duan2020minimax}).  
It implies the bound in Theorem \ref{thm:upper_bound_eva2} is nearly minimax-optimal.
\begin{theorem}[{\bf Minimax lower bound for sparse policy evaluation}]\label{thm:lower_bound_pe}
	Suppose Assumption \ref{assm:data_colle} holds.
	If $N \gtrsim sL(1-\gamma)^{-1}$, then 
\begin{equation*}
\begin{split}
     \inf_{\hat{v}^{\pi}}&\sup_{\phi,M\in \cM_{\phi,s}(\cX,\cA,\gamma)}\mathbb P_{M}\Big(|\hat{v}^{\pi}(\cD)-v^{\pi}|\gtrsim \frac{1}{(1-\gamma)^2} \sqrt{1+\chi^2_{\cG(\cK)}(\mu^{\pi}, \bar{\mu})}\sqrt{\frac{1}{N}}\Big)\geq \frac{1}{6}\,,
\end{split}
\end{equation*}
where $\phi\in (\RR^d)^{\cX \times \cA}$, 
$\mathbb P_{M}$ is the probability measure under the DMDP instance $M$
and  $\hat{v}^{\pi}(\cdot)$ sweeps through all algorithms that estimate values based on data $\cD$. 
\end{theorem}

\begin{remark}
It is worth to mention that in Theorem \ref{thm:lower_bound_pe}, the distribution mismatch term $1+\chi^2_{\cG(\cK)}(\mu^{\pi}, \bar{\mu})$ may also contain a $1-\gamma$ term in the worse case. Thus the lower bound of sparse off-policy policy evaluation also has a $\sqrt{1/(1-\gamma)^3}$ dependency in the worse case that matches the result for the lower bound of sparse batch policy optimization in Theorem \ref{thm:lowerbound}.
\end{remark}

\section{Sparsity-Aware Batch Policy Optimization}

We extend our analysis to batch policy learning problem for sparse linear DMDPs. Consider the Lasso fitted Q-iteration (see Algorithm \ref{alg:batch_opt}) that has been studied in \cite{calandriello2014sparse} as a special case of an algorithm for sparse multi-task RL.
It resembles Algorithm \ref{alg:batch_eva} 
except for that it calculates the regression target with an additional ``max" operation. The next theorem proves the approximate optimality of the learned policy using Lasso fitted Q-iteration. 

 {\small
\begin{algorithm}[htb!]
	\caption{Lasso-regularized fitted Q-iteration \citep{calandriello2014sparse}}
	\begin{algorithmic}[1]\label{alg:batch_opt}
		\STATE
		\textbf{Input:} $T$ folds of dataset $\{\cD_t\}_{t=1}^T$, regularization parameter $\lambda_1$, $\hat{w}_0= 0\in \RR^d$.
		\STATE
		\textbf{Repeat:}
		\FOR{$t=1,2,\ldots, T$}
		\STATE Calculate regression targets: 	for $(x_i, a_i, x_i')\in\cD_t$,
		$
		y_i = \max_{a\in\cA} Q_{\hat{w}_{t-1}}(x'_i, a)\,.
		$
		\STATE Based on $\{(\Pi_{[0,1/(1-\gamma)]}y_i, \phi(x_i, a_i))\}_{(x_i, a_i, x'_i)\in\cD_t}$, fit $\hat{w}_{t}$ through Lasso as in Algorithm \ref{alg:batch_eva}.
		\ENDFOR
			\STATE
		\textbf{Output:} policy $\hat{\pi}_T(\cdot|x) = \max_{a\in\cA} Q_{\hat{w}_T}(x, a), \forall x\in\cX$.
	\end{algorithmic}
\end{algorithm}
}

\begin{theorem}[{\bf Approximate optimality of the learned policy}]
	\label{thm:agm-result}
 Suppose Assumptions \ref{assum:sparse_MDP}, \ref{assm:data_colle} hold and $C_{\min}(\Sigma, s)>0$. 
	Let $N\gtrsim s^2L(1-\gamma)^{-1}/C_{\min}(\Sigma, s).$ Let
	Algorithm~\ref{alg:batch_opt} take $T = \Theta(\log (N/(1-\gamma))/(1-\gamma))$ and  $\lambda_1=(1-\gamma)^{-1}\sqrt{T\log (2d/\delta)/N}$.  
	Then, with probability at least $1-\delta$, 
	\begin{equation}\label{eqn:bound_policy_optimization}
	     \big\|v^{\hat{\pi}_{T}}-v^*\big\|_{\infty}\lesssim \frac{s}{C_{\min}(\Sigma, s)}\sqrt{\frac{\log(d/\delta)}{N(1-\gamma)^7}}.
	\end{equation}
\end{theorem}

Theorem \ref{thm:agm-result} suggests that the sample size needed to get a good policy depends mainly on the number of relevant features $s$, instead of the large ambient dimension $d$, provided that the data is well-conditioned. This result is not surprising: \citet{calandriello2014sparse} gave a similar upper bound for sparse FQI for the setting of generative model. \citet{le2019batch} provided a generalization theory for policy evaluation/learning with a general function class and their error bound depends on the VC-dimension of the class, but it requires a stronger coefficient concentration condition. 

In the end, we study the fundamental limits of sparse batch policy learning. We establish an information-theoretic minimax lower bound that nearly match the aforementioned upper bound.

\begin{theorem}[{\bf Minimax lower bound for batch policy optimization in sparse linear DMDPs}] \label{thm:lowerbound}

	Suppose Assumption \ref{assm:data_colle} holds and $\gamma \geq \frac{2}{3}$. Let $\widehat{\pi}$ denote an algorithm that maps dataset $\mathcal{D}$ to a policy $\widehat{\pi}(\mathcal{D})$. If $N \gtrsim sL(1-\gamma)^{-1}$, then for any $\widehat{\pi}$, there always exists a DMDP instance $M \in \mathcal{M}_{\phi,s}(\mathcal{X}, \mathcal{A}, \gamma)$ with feature $\phi \in (\mathbb{R}^d)^{\mathcal{X} \times \mathcal{A}}$ satisfying $\|\phi(x,a)\|_{\infty} \leq 1$ for all $(x,a) \in \mathcal{X} \times \mathcal{A}$,
	such that
	\[ 
	\mathbb{P}_{M} \Big( v_{\xi_0}^{*} - v_{\xi_0}^{\widehat{\pi}(\mathcal{D})} \gtrsim \sqrt{\frac{s }{C_{\min}(\Sigma,s)}} \sqrt{\frac{1}{N(1-\gamma)^3}} \Big) \geq \frac{1}{6}. \]
\end{theorem}
Theorems \ref{thm:agm-result}, \ref{thm:lowerbound} show that the statistical error of batch policy learning is fundamentally determined by the ratio $s/C_{\min}(\Sigma,s)$. Note that there remains a gap $\sqrt{s/C_{\min}(\Sigma,s)}$ between Theorems \ref{thm:agm-result} and \ref{thm:lowerbound}, due to the nature of Lasso regression.  
\begin{remark}[\bf Role of the minimal eigenvalue and well-conditioned data.]
Earlier results such as those of \cite{munos2008finite,antos2008fitted,le2019batch} require stronger forms of concentration condition that the state-action occupancy measure (or a ratio involving this measure) is entrywisely bounded across all policies. 
Such entrywise bound can be very large if the state-action space $\cX$ is large. 
In contrast, our results only require that the data's covariance $\Sigma$ is well-conditioned on restricted supports, which is a much weaker assumption. 
Further, one can use the empirical minimal eigenvalue to get a rough error estimate. 
Theorem \ref{thm:lowerbound} further validates that the minimal eigenvalue indeed determines the statistical limit of batch policy optimization. The result is the first of its kind to our best knowledge.
\end{remark}

\section{Experiment}
In this section, we conducted some preliminary experiments with a Mountain Car \citep{moore1990efficient} example to demonstrate the advantage of sparse learning in OPE problem. We use 800 radial basis functions for linear value function approximation and compare our Lasso-FQE with the standard FQE. We pick the random policy as the behavior one and a near-optimal policy as the target, and we measure the estimation error by $|\hat{v}-v^*|/|v^*|$. We constructed multiple behavior policies with varying levels of $\epsilon$-greedy noise, and plot their OPE error against their (restricted) $\chi^2$-divergence from the target policy. The results are averaged by 20 runs and summarized in Figure \ref{fig:mountain_car} in the appendix. It shows that our Lasso-FQE clearly has smaller estimation error compared with FQE, proving the sparse feature selection is effective in a practical RL example, and demonstrates how the distribution mismatch ($\chi^2$-divergence term) affects OPE error (with sample size fixed).  The results confirm our theorems that the (restricted) chi-square divergence sharply determines the (sparse) OPE error.

\section{Conclusion}
In this work we focus on high-dimensional batch RL using sparse linear function approximation. While previous work in RL recognized the possibility of bringing tools from sparse learning to RL, they lacked a clean theoretical framework and formal results. By building on the strength of the linear DMDP framework, our result show that learning and planning in linear DMDPs can be done in the ``feature space'' even in the presence of sparsity and when only batch data is available. 
\bibliographystyle{plainnat}
{\small
\bibliography{ref}
}

\setlength{\footskip}{55pt}

\clearpage
\onecolumn
\appendix

The appendix is organized as follows.
\begin{itemize}
\item Appendix \ref{sec:linmdpproofs}. Proofs of Propositions \ref{prop:linmdp} and \ref{prop:linmdpf}: properties of linear MDPs.
\item Appendix \ref{sec:main_results}. Proofs of results about off-policy policy evaluation in Section \ref{sec:bound_OPE}.
\begin{itemize}
    \item Appendix \ref{proof:upper_bound_eva1}. Proof of Theorem \ref{thm:upper_bound_eva1}.
    \item Appendix \ref{sec:proof_group_lasso}. Proof of Lemma \ref{lemma:group_lasso}.
    \item Appendix \ref{proof:upper_bound_eva2}. Proof of Theorem \ref{thm:upper_bound_eva2}.
\end{itemize}
\item Appendix \ref{proof:upper_bound}.  Proof of Theorem \ref{thm:agm-result}: upper bound of batch policy optimization.
    \item Appendix \ref{sec:proof_lower_bound_BPO}. Proof of Theorem \ref{thm:lowerbound}: minimax lower bound of batch policy optimiztion.
    \item Appendix \ref{sec:auxiliary}. Proofs of auxiliary lemmas.
    \item Appendix \ref{sec:supporting}. Some supporting lemmas.
    \item Appendix \ref{sec:exper}. Results of experiments.
\end{itemize}

\section{Proofs concerning linear MDPs}\label{sec:linmdpproofs}
\begin{proof}[Proof of Proposition~\ref{prop:linmdp}]
Let $Q_\pi = r+\gamma P v_\pi$. 
Then, $v_\pi = Q_\pi^\pi$, and hence
since $M$ is a linear MDP, $P Q_\pi^\pi = g_w$ with some $w\in \RR^d$.
Hence, $P v_\pi = g_w$.
Furthermore, by its definition
thus $Q_\pi = r+\gamma g_w$. Then, applying $(\cdot)^\pi$ on both sides
and using again $Q_\pi^\pi=v_\pi$, we get
$v_\pi = r^\pi + \gamma g_w^\pi$, which finishes the proof of the first part.

Now, if $w$ is as above,
$g_w 
= P v_\pi 
= P( r^\pi + \gamma (P v_\pi)^\pi)
= P( r^\pi + \gamma g_w^\pi)
= P r^\pi + \gamma P g_w^\pi$.

Finally, assuming that $w$ satisfies the last identity, defining $Q = r +\gamma g_w$, we have
$Q 
= r + \gamma (Pr^\pi + \gamma P g_w^\pi)
= r + \gamma (P(r + \gamma g_w)^\pi)
= r + \gamma P Q^\pi$.
As is well known, the unique fixed point of this equation is $Q_\pi$. Hence, $Q=Q_\pi$.
\end{proof}

\begin{proof}[Proof of Proposition \ref{prop:linmdpf}]
Fix a policy $\pi$.
Since $M$ is a linear MDP, for every $i\in [d]$ there exist $w_i\in \RR^d$ such that $P \phi_i^\pi = g_{w_i}$.
Thus, for any $(x,a)\in \cX \times \cA$,
\begin{align*}
\mathbb E_{x'\sim P(\cdot|x,a)}[\phi^{\pi}(x')^{\top}|x,a]  
 & = (P \phi_1^\pi ,\dots, P \phi_d^\pi ) 
 = (\phi(x,a)^\top w_1, \dots, \phi(x,a)^\top w_d) \\
 &= \phi(x,a)^\top \begin{pmatrix} w_1, \cdots, w_d \end{pmatrix}\,.
\end{align*}
Thus, \cref{eqn:identity} holds if we choose
\begin{align*}
K^{\pi} =  \begin{pmatrix} w_1, \cdots, w_d \end{pmatrix}\,.
\end{align*} 
Under the sparsity assumption, Assumption \ref{assum:sparse_MDP},
there exists $\cK\subset [d]$ such that $w_{ij}=0$ when $j\not\in \cK$. This shows that all but $|\cK|$ rows of $K^{\pi}$ are identically zero, finishing the proof.
\end{proof}

\section{Proofs of off-policy policy evaluation}\label{sec:main_results}

\subsection{Proof of Theorem \ref{thm:upper_bound_eva1}: lasso fitted Q-evaluation}\label{proof:upper_bound_eva1}

Recall that we split the whole dataset $\cD$ into $T$ folds and each fold consists of $R$ episodes or $RL$ sample transitions.  At $t$th phase, only the fresh fold of dataset $\cD_t = \{(x_i^{(t)}, a_i^{(t)},  x_i^{(t)'})\}_{i=1}^{RL}$ is used. 

\textbf{Step 1: Approximate value iteration.} We first show that the execution of Algorithm \ref{alg:batch_eva} is equivalent to approximate value iteration. Denote a Lasso estimator with respect to a function $V$ at $t$th phase: 
\begin{equation}\label{eqn:sparse_solution}
\hat{w}_{t}(V) = \argmin_{w\in\mathbb R^d}\Big(\frac{1}{RL}\sum_{i=1}^{RL}\Big(V(x_i^{(t)'}) -\phi(x_i^{(t)},a_i^{(t)})^{\top}w\Big)^2 + \lambda_1 \|w\|_1\Big).
\end{equation}
Note that $\hat{w}_{t}(\cdot)$ only depends data collected at the $t$th phase. Define the parameterized value function as
\begin{equation*}
V_{w}^{\pi}(x)= \sum_{a}\pi(a|x) \big(r(x,a) + \gamma\phi(x,a)^{\top}w\big).    
\end{equation*}
Define an approximate Bellman operator for target policy $\pi$, i.e. $\hat{\cT}_{\pi}^{(t)}:\mathbb R^{\cX}\to\mathbb R^{\cX}$ as:
\begin{align}\label{eqn:ABO}
    [\hat{\cT}_{\pi}^{(t)}V](x) & := V_{\hat w_t(V)}^{\pi}(x) 
   = 
    \sum_a \pi(a|x)\Big(r(x,a) + \gamma\phi(x,a)^{\top}\hat{w}_t(V)\Big) . 
\end{align}
Note this $\hat{\cT}_{\pi}^{(t)}$ is a randomized operator that only depends data in the $t$th fold. 
It is easy to see that if $(\hat w_t)_{t=1}^T$ is the sequence of weights computed in  Algorithm \ref{alg:batch_eva} 
then
$\hat w_t = \hat{w}_t(\Pi_{[0,1/(1-\gamma)]} V_{\hat w_{t-1}}^\pi)$ and also
\begin{align}
V^{\pi}_{\hat w_t} = \hat{\cT}_{\pi}^{(t)}\Pi_{[0,1/(1-\gamma)]}V^{\pi}_{\hat{w}_{t-1}}\,.
\label{eqn:update}
\end{align}

\textbf{Step 2: Linear representation of Bellman operator.} 
Recall that the true Bellman operator for target policy $\pi$, i.e.  $\cT_{\pi}:\mathbb R^{\cX}\to\mathbb R^{\cX}$ is defined as 
\begin{equation}\label{eqn:true_bell_pi}
    [\cT_{\pi} V](x) := \sum_a \pi(a|x)\Big(r(x,a)+\gamma \sum_{x'}P(x'|x,a)V(x')\Big).
\end{equation}
We first verify for each phase $t\in[T]$, $\cT_{\pi}\Pi_{[0,1/(1-\gamma)]}V^{\pi}_{\hat{w}_{t}}$ has a linear representation. From Assumption \ref{assum:sparse_MDP}, it holds that
\begin{equation*}
\begin{split}
     [\cT_{\pi}\Pi_{[0,1/(1-\gamma)]}V^{\pi}_{\hat{w}_{t}}](x) &= \sum_{a}\pi(a|x)\Big(r(x,a) + \gamma \sum_{x'}\Pi_{[0,1/(1-\gamma)]}V^{\pi}_{\hat{w}_{t}}(x')P(x'|x,a)\Big)\\
     &= \sum_{a}\pi(a|x)\Big(r(x,a) + \gamma \sum_{x'}\Pi_{[0,1/(1-\gamma)]}V^{\pi}_{\hat{w}_{t}}(x')\phi(x,a)^{\top}\psi(x')\Big)\\
     &=\sum_{a}\pi(a|x)\Big(r(x,a) + \gamma \phi(x,a)^{\top} \sum_{x'}\Pi_{[0,1/(1-\gamma)]}V^{\pi}_{\hat{w}_{t}}(x')\psi(x')\Big).
\end{split}
\end{equation*}
For a vector $\bar{w}_t\in\mathbb R^d$, we define its $k$th coordinate as
\begin{equation}\label{eqn:def_w_bar1}
    \bar{w}_{t,k} = \sum_{x'}\Pi_{[0,1/(1-\gamma)]}V^{\pi}_{\hat{w}_{t-1}}(x')\psi_k(x'), \ \text{if} \ k\in\cK,
\end{equation}
and $\bar{w}_{t,k}=0$ if $k\notin \cK$. Then we have 
\begin{equation}\label{eqn:linear_true}
    [\cT_{\pi}\Pi_{[0,1/(1-\gamma)]}V^{\pi}_{\hat{w}_{t}}](x)=\sum_{a}\pi(a|x)\Big(r(x,a) + \gamma \phi(x,a)^{\top} \bar{w}_{t}\Big) .
\end{equation}
It shows that $\cT_{\pi}\Pi_{[0,1/(1-\gamma)]}V^{\pi}_{\hat{w}_{t}}$ has a linear representation if the reward could also be linearly represented. For notation simplicity, we drop the supscript of $x_i^{(t)}$ and $a_i^{(t)}$ for the following derivations when there is no ambiguity.

\textbf{Step 3: Sparse linear regression.} We interpret $\bar{w}_t$ as the ground truth of the lasso estimator in Algorithm \ref{alg:batch_eva} at phase $t$,  in terms of the following sparse linear regression:
\begin{eqnarray}\label{eqn:regression_eva}
    \Pi_{[0,1/(1-\gamma)]}V^{\pi}_{\hat{w}_{t-1}}(x_i') = \phi(x_i, a_i)^{\top}\bar{w}_t+\varepsilon_i, i=1\ldots, RL,
\end{eqnarray}
where $\varepsilon_i =  \Pi_{[0,1/(1-\gamma)]}V^{\pi}_{\hat{w}_{t-1}}(x_i') -\phi(x_i, a_i)^{\top}\bar{w}_t$. Define a filtration $\{\cF_i\}_{i=1,\ldots, RL}$ with $\cF_i$ generated by $\{(x_1, a_1),\ldots, (x_{i}, a_{i})\}$. By the definition of $V_{\hat{w}_{t-1}}$ and $\bar{w}_t$ in Eq.~\eqref{eqn:def_w_bar1}, we have 
\begin{equation*}
    \begin{split}
        \mathbb E[\varepsilon_{i}|\cF_i] &=   \mathbb E\big[\Pi_{[0,1/(1-\gamma)]}V^{\pi}_{\hat{w}_{t-1}}(x_i')|\cF_i\big] -\phi(x_i, a_i)^{\top}\bar{w}_t\\
        &= \sum_{x'}[\Pi_{[0,1/(1-\gamma)]}  V^{\pi}_{\hat{w}_{t-1}}](x')P(x'|x_i, a_i) '-\phi(x_i, a_i)^{\top}\bar{w}_t\\
        & = \sum_{k\in\cK}\phi_k(x_i, a_i)\sum_{x'}[\Pi_{[0,1/(1-\gamma)]}V^{\pi}_{\hat{w}_{t-1}}](x')\psi_k(x')'-\phi(x_i, a_i)^{\top}\bar{w}_t = 0.
    \end{split}
\end{equation*}
Therefore, $\{\varepsilon_i\}_{n=1}^{RL}$ is a sequence of martingale difference noises and $|\varepsilon_i|\leq 1/(1-\gamma)$ due to the truncation operator $\Pi_{[0,1/(1-\gamma)]}$. 
The next lemma bounds the difference between $\hat{w}_t$ and $\bar{w}_t$ within $\ell_1$-norm.  
The proof is deferred to Appendix \ref{sec:proof_lasso_l1}.

\begin{lemma}\label{lemma:lasso_l1_bound}
Consider the sparse linear regression described in Eq.~\eqref{eqn:regression_eva}. Suppose the restricted minimum eigenvalue of $\Sigma$ satisfy $C_{\min}(\Sigma, s)>0$ and the number of episodes used in phase $t$ satisfies
\begin{equation*}
    R\geq \frac{C_1\log(3d^2/\delta)s^2}{C_{\min}(\Sigma, s)},
\end{equation*}
for some absolute constant $C_1>0.$ With the choice of $\lambda_1=(1-\gamma)^{-1}\sqrt{\log (2d/\delta)/(RL)}$, the following holds with probability at least $1-\delta$, 
\begin{equation}\label{eqn:lasso_error_bound}
  \big\|\hat{w}_t-\bar{w}_t\big\|_1\leq \frac{16\sqrt{2}s}{C_{\min}(\Sigma, s)}\frac{1}{1-\gamma}\sqrt{\frac{\log(2d/\delta)}{RL}}.
\end{equation}
\end{lemma}
Note that the samples we use between phases are mutually independent. Thus, Eq.~\eqref{eqn:lasso_error_bound} uniformly holds for all $t\in[T]$ with probability at least $1-T\delta$.

\textbf{Step 4: Error decomposition.}
Recall that $\hat{v}_{\hat{w}_T}^{\pi}= \frac1m\sum_{u=1}^m  \Pi_{[0,1/(1-\gamma)]} (Q_{\hat{w}_T}(\tilde{x}_u,\tilde{a}_u))$ and we denote $\bar{v}^{\pi}_{\hat{w}_T} = \sum_{x}V^{\pi}_{\hat{w}_T}(x) \xi_0(x).$ According to Eq.~\eqref{eqn:update}, we decompose the policy evaluation error by Monte Carlo error, estimation error and approximation error as follows:
\begin{equation}\label{eqn:error_decomp}
\begin{split}
     |\hat{v}^{\pi}_{\hat{w}_T}-v^{\pi}|
     =& \Big|\hat{v}_{\hat{w}_T}^{\pi} - \bar{v}^{\pi}_{\hat{w}_T} + \sum_x\Big(V^{\pi}_{\hat{w}_T}(x) - v^{\pi}(x)\Big)\xi_0(x) \Big|\\
      =& \Big|\hat{v}_{\hat{w}_T}^{\pi} - \bar{v}^{\pi}_{\hat{w}_T} +\sum_x\Big([\hat{\cT}_{\pi}^{(T)}\Pi_{[0,1/(1-\gamma)]}V^{\pi}_{\hat{w}_{T-1}}](x) - [\cT_{\pi}v^{\pi}](x)\Big)\xi_0(x) \Big|\\
     \leq& \Big|\underbrace{\sum_x\Big([\hat{\cT}_{\pi}^{(T)}\Pi_{[0,1/(1-\gamma)]}V^{\pi}_{\hat{w}_{T-1}}](x)- [\cT_{\pi}\Pi_{[0,1/(1-\gamma)]}V^{\pi}_{\hat{w}_{T-1}}](x)\Big)\xi_0(x) }_{\text{estimation error}} \Big|\\
     &+ \Big| \underbrace{\sum_x\big[\cT_{\pi}(\Pi_{[0,1/(1-\gamma)]}V^{\pi}_{\hat{w}_{T-1}} - v^{\pi})\big](x)\xi_0(x)}_{\text{approximation error}}\Big| + \underbrace{\Big|\hat{v}_{\hat{w}_T}^{\pi} - \bar{v}^{\pi}_{\hat{w}_T}\Big|}_{\text{Monte Carlo error}}.
\end{split}
\end{equation}

Since $\tilde{x}_u, \tilde{a}_u$ is i.i.d sampled from $\xi_0$ and $\pi$, standard Hoeffding's inequality shows that
\begin{equation}\label{eqn:mc_error}
    \Big|\hat{v}_{\hat{w}_T}^{\pi} - \bar{v}^{\pi}_{\hat{w}_T}\Big|\leq \sqrt{\frac{\log(1/\delta)}{m(1-\gamma)^2}} = \sqrt{\frac{\log(1/\delta)}{N(1-\gamma)^2}},
\end{equation}
with probability at least $1-\delta$.

Recall that $\nu_t^{\pi} = \mathbb E^{\pi}[\phi(x_t,a_t)|x_0\sim\xi_0]$. To bound the estimation error, combining Eqs.~\eqref{eqn:ABO} and \eqref{eqn:linear_true} together, we have
\begin{equation}\label{eqn:diff_bell}
\begin{split}
     &\sum_x\Big([\hat{\cT}_{\pi}^{(T)}\Pi_{[0,1/(1-\gamma)]}V^{\pi}_{\hat{w}_{T-1}}](x) -[\cT_{\pi}\Pi_{[0,1/(1-\gamma)]}V^{\pi}_{\hat{w}_{T-1}}](x)\Big)\xi_0(x) \\
     =& \gamma\sum_x\sum_{a}\pi(a|x)\phi(x,a)^{\top}(\hat{w}_{T-1}-\bar{w}_{T-1})\xi_0(x)\\
     =& \gamma(\nu_0^{\pi})^{\top}(\hat{w}_{T-1}-\bar{w}_{T-1}).
\end{split}
\end{equation}
To bound approximation error, we expand it by Eq.~\eqref{eqn:true_bell_pi}:
\begin{equation*}
    \begin{split}
        &\sum_x\big[\cT_{\pi}(\Pi_{[0,1/(1-\gamma)]}V^{\pi}_{\hat{w}_{T-1}} - v^{\pi})\big](x)\xi_0(x)\\
        =&\gamma \sum_x\Big(\sum_a \pi(a|x)   \sum_{x'}P(x'|x,a)\big(\Pi_{[0,1/(1-\gamma)]}V^{\pi}_{\hat{w}_{T-1}} - v^{\pi}\big)(x')\Big)\xi_0(x).
    \end{split}
\end{equation*}
According to Eq.~\eqref{eqn:diff_bell}, we decompose
\begin{equation*}
\begin{split}
    &\big(\Pi_{[0,1/(1-\gamma)]}V^{\pi}_{\hat{w}_{T-1}} - v^{\pi}\big)(x)\leq (V^{\pi}_{\hat{w}_{T-1}} - v^{\pi})(x)\\
    =& \Big[\hat{\cT}_{\pi}^{(T-1)}\Pi_{[0,1/(1-\gamma)]}V^{\pi}_{\hat{w}_{T-2}} -\cT_{\pi}\Pi_{[0,1/(1-\gamma)]}V^{\pi}_{\hat{w}_{T-2}} +\cT_{\pi}\Pi_{[0,1/(1-\gamma)]}V^{\pi}_{\hat{w}_{T-2}} -\cT_{\pi}v^{\pi}\Big](x)\\
    =& \gamma\sum_{a}\pi(a|x)\phi(x,a)^{\top}(\hat{w}_{T-2} - \bar{w}_{T-2}) + \big[\cT_{\pi}(\Pi_{[0,1/(1-\gamma)]}V^{\pi}_{\hat{w}_{T-2}}-v^{\pi})\big](x).
\end{split}
\end{equation*}
This implies
\begin{equation}\label{eqn:bound2}
    \begin{split}
        &\sum_x\big[\cT_{\pi}(\Pi_{[0,1/(1-\gamma)]}V^{\pi}_{\hat{w}_{T-1}} - v^{\pi})\big](x)\xi_0(x)\\
        &\leq \gamma^2\sum_x\Big(\sum_a\pi(a|x)\gamma\sum_{x'}P(x'|x,a)\big(\sum_{a}\pi(a|x')\phi(x',a)^{\top}(\hat{w}_{T-2} - \bar{w}_{T-2})\big)\Big)\xi_0(x)\\
        &+\sum_x\Big(\sum_a\pi(a|x) \gamma\sum_{x'}\big[\cT_{\pi}(\Pi_{[0,1/(1-\gamma)]}V^{\pi}_{\hat{w}_{T-2}}-v^{\pi})\big](x')P(x'|x,a)\Big)\xi_0(x)\\
        &= \gamma^2 \mathbb E^{\pi}[\phi(x_1,a_1)|x\sim\xi_0]^{\top}(\hat{w}_{T-2} - \bar{w}_{T-2}) \\
        &+ \gamma\mathbb E^{\pi}\Big[[\cT_{\pi}(\Pi_{[0,1/(1-\gamma)]}V^{\pi}_{\hat{w}_{T-2}}-v^{\pi})](x_1)|x_0\sim \xi_0\Big].
    \end{split}
\end{equation}
Combining Eqs.~\eqref{eqn:error_decomp}, \eqref{eqn:diff_bell} and \eqref{eqn:bound2} together, we have 
\begin{equation*}
\begin{split}
    |\bar{v}^{\pi}_{\hat{w}_T}-v^{\pi}| \leq& |\gamma(\nu_0^{\pi})^{\top}(\hat{w}_{T-1}-\bar{w}_{T-1})| + |\gamma^2 (\nu_1^{\pi})^{\top}(\hat{w}_{T-2}-\bar{w}_{T-2})|\\
    &+ \Big|\gamma\mathbb E^{\pi}\Big[[\cT_{\pi}(\Pi_{[0,1/(1-\gamma)]}V^{\pi}_{\hat{w}_{T-2}}-v^{\pi})](x_1)|x_0\sim \xi_0\Big]\Big|.
\end{split}
\end{equation*}
Iteratively implementing the above decomposition, we have
\begin{equation*}
\begin{split}
    |\bar{v}^{\pi}_{\hat{w}_T}-v^{\pi}| &\leq \sum_{t=0}^{T-1}\gamma^{t+1}|(\nu_t^{\pi})^{\top}(\hat{w}_{T-t-1}-\bar{w}_{T-t-1})| + \gamma^{T} \Big| \mathbb E^{\pi}\Big[(\Pi_{[0,1/(1-\gamma)]}V^{\pi}_{\hat{w}_{0}}-v^{\pi})(x_{T})|x_0\sim \xi_0\Big]\Big|\\
    &\leq  \sum_{t=0}^{T-1}\gamma^{t+1}|(\nu_t^{\pi})^{\top}(\hat{w}_{T-t-1}-\bar{w}_{T-t-1})|+  \frac{2\gamma^{T} }{1-\gamma}\\
    &\leq  \sum_{t=0}^{T-1}\gamma^{t+1}\|\nu_t^{\pi})\|_{\infty}\|\hat{w}_{T-t-1}-\bar{w}_{T-t-1}\|_1 + \frac{2\gamma^{T} }{1-\gamma}.
\end{split}
\end{equation*}
Since we assume $\|\phi(x,a)\|_{\infty}\leq 1$, then $\|\nu_t^{\pi}\|_{\infty}\leq 1$ as well. Using the fact that  $\sum_{t=0}^{T-1} \gamma^t\leq 1/(1-\gamma)$, we have
\begin{equation*}
     |\bar{v}^{\pi}_{\hat{w}_T}-v^{\pi}| \leq \frac{1}{1-\gamma}\max_{t=0, \ldots, T-1}\|\hat{w}_{t}-\bar{w}_{t}\|_1 + \frac{2\gamma^{T} }{1-\gamma}.
\end{equation*}

Suppose the sample size satisfies 
\begin{equation*}
    N\geq \frac{C_1\log(3d^2/\delta)s^2}{C_{\min}(\Sigma, s)}\frac{L}{1-\gamma}\log(N/(1-\gamma))),
\end{equation*}
for a sufficient large constant $C_1>0$.
Applying Lemma \ref{lemma:lasso_l1_bound} over $t=0, \ldots, T-1$,
it implies 
\begin{equation*}
    |\bar{v}^{\pi}_{\hat{w}_T}-v^{\pi}| \leq\frac{1}{(1-\gamma)^2} \frac{16\sqrt{2}s}{C_{\min}(\Sigma, s)}\sqrt{\frac{\log(2d/\delta)}{RL}} +  \frac{2\gamma^{T}}{1-\gamma},
\end{equation*}
holds with probability at least $1-T\delta$. By elementary change of base formula and Taylor expansion, we have 
\begin{equation*}
    \log_{1/\gamma}(x) = \frac{\log(x)}{\log(1/\gamma)}\approx \frac{\log(x)\lambda}{1-\gamma}.
\end{equation*}
By properly choosing $T = \Theta(\log (N/(1-\gamma))/(1-\gamma))$, we have with probability at least $1-\delta$,
\begin{equation*}
       |\bar{v}^{\pi}_{\hat{w}_T}-v^{\pi}| \leq \frac{1}{(1-\gamma)^{5/2}} \frac{32\sqrt{2}s}{C_{\min}(\Sigma, s)}\sqrt{\frac{\log (N/(1-\gamma))\log (2dT/\delta)}{N}},
\end{equation*}
where we use $N = TRL$. Combining with Monte Carlo approximation error Eq.~\eqref{eqn:mc_error} This ends the proof.   \hfill $\blacksquare$\\

\subsection{Proof of Lemma \ref{lemma:group_lasso}: feature selection}\label{sec:proof_group_lasso}
We study the feature screening and sparsity properties of the model selected by the regularized estimator $\hat{K}^{\pi}$. Recall that from the identity Eq.~\eqref{eqn:identity}, we solve the following multivariate regression problem:
\begin{equation}\label{eqn:multi_reg_ori}
    \phi^{\pi}(x_n')^{\top} = \phi(x_n, a_n)^{\top} K^{\pi} + \varepsilon_n, \ n\in[N],
\end{equation}
where $x_n'\sim P(\cdot|x_n, a_n)$ and $\varepsilon_n = \phi^{\pi}(x_n')^{\top} - \mathbb E[\phi^{\pi}(x_n')^{\top}]\in\mathbb R^d$ is the noise vector. Define a filtration $\{\cF_n\}_{n=1,\ldots, N}$ with $\cF_n$ generated by $\{(x_1, a_1),\ldots, (x_n, a_n)\}$. It is easy to see $\mathbb E[\varepsilon_{n}|\cF_n] =0$ such that $\{\varepsilon_n\}_{n=1}^N$ are martingale difference vectors. 
We introduce some notations for simplicities:
\begin{itemize}
    \item Denote the response $Y_j = (\phi^{\pi}_j(x_1'), \ldots, \phi^{\pi}_j(x_N'))^{\top}$ for $j\in[d]$ where $\phi^{\pi}_j(\cdot)$ is $j$th coordinate of $\phi^{\pi}(\cdot)$. And $Y=(Y_1^{\top}, \ldots, Y_d^{\top})^{\top}\in\mathbb R^{Nd\times 1}$. 
    \item Denote the noise $W_j = (\varepsilon_{1j}, \ldots, \varepsilon_{Nj})^{\top}$ where $\varepsilon_{nj}$ is the $j$th coordinate of $\varepsilon_n$, and $W = (W_1^{\top}, \ldots, W_d^{\top})^{\top}\in\mathbb R^{Nd\times 1}$.
    \item Denote the design matrix as 
  $$  \begin{matrix}
\Phi=\begin{pmatrix}
\phi(x_1,a_1)^{\top} \\ 
\vdots\\
\phi(x_N,a_N)^{\top}
\end{pmatrix}\in\mathbb R^{N\times d};   X = \begin{pmatrix}
\Phi & \ldots& 0 \\ 
\vdots&\ddots & \vdots\\
0&\ldots&\Phi
\end{pmatrix}\in\mathbb R^{Nd\times d^2}.
\end{matrix}$$
Note that $X$ is a block diagonal matrix.
\item Let $\bbeta_j^*$ as the $j$th column of $K^{\pi}$ and $\bbeta^* = (\bbeta_1^{*\top}, \ldots, \bbeta_d^{*\top})^{\top}\in\mathbb R^{d^2\times 1}$ as the regression coefficient.
\item For every $\bbeta\in\mathbb R^{d^2}$, we define $\bbeta^j = (\beta_{j+(i-1)d}:i\in[d])^{\top}$ as the vector formed by the coefficients corresponding to the $j$th variable. For instance, $\bbeta^{*j}$ is the $j$th row of $K^{\pi}$. If $J\subseteq [d]$, denote $\bbeta_J\in\mathbb R^{d^2}$ by stacking the vectors $\bbeta^j\ind\{j\in J\}$. Write $\cS(\bbeta) = \{j:\bbeta^j\neq 0\}$ as the relevant feature set of $\bbeta$. 
\item For a vector $\bbeta\in\mathbb R^{d^2}$, define the $\ell_{2,p}$-norm for $1\leq p<\infty$ as:
\begin{equation*}
    \|\bbeta\|_{2,p} =  \Big(\sum_{j=1}^d\Big(\sum_{i=1}^d\beta_{j+(i-1)d}^2\Big)^{p/2}\Big)^{1/p},
\end{equation*}
and the $\ell_{2,0}$-norm as:
\begin{equation*}
    \|\bbeta\|_{2,0} = \sum_{j=1}^d\ind\big\{\|\bbeta^j\|_2\neq 0\big\}.
\end{equation*}
\end{itemize}
Therefore, we can rewrite Eq.~\eqref{eqn:multi_reg_ori} into an ordinary linear regression form with group sparse structure on the regression coefficients $\bbeta^*$:
\begin{equation*}
    Y = X\bbeta^* + W.
\end{equation*}
Note that $\cS(\bbeta^*) = \cK$ where $\cK$ is defined in Assumption \ref{assum:sparse_MDP} since $K^{\pi}$ is row-sparse. The corresponding group lasso estimator defined in Eq.~\eqref{eqn:group_lasso} can be rewritten into:
\begin{equation}\label{eqn:group_lasso_rewrite}
    \hat{\bbeta} = \argmin_{\bbeta}\Big\{\frac{1}{Nd}\big\|Y-X\bbeta \big\|_2^2+\lambda_2 \sum_{j=1}^d\|\bbeta^j\|_2\Big\},
\end{equation}
and $\cS(\hat{\bbeta}) = \hat{\cK}$. The regularization parameter is chosen as
\begin{equation}\label{eqn:lambda_2}
    \lambda_2 =4\sqrt{\frac{2\log(2d^2/\delta)}{Nd}},
\end{equation}
for some $\delta>0$.

Now we study the feature screening property of $\hat{\bbeta}$ in four steps. 

\textbf{Step 1.} By the optimality of $\hat{\bbeta}$, we have 
\begin{equation*}
    \frac{1}{Nd}\|Y -X\hat{\bbeta}\|_2^2+\lambda_2 \sum_{j=1}^d\|\hat{\bbeta}^j\|_2\leq  \frac{1}{Nd}\|Y -X\bbeta^*\|_2^2+\lambda_2 \sum_{j=1}^d\|\bbeta^{*j}\|_2.
\end{equation*}
Plugging in $Y = X\bbeta^*+W$, 
\begin{equation}\label{eqn:decom}
    \begin{split}
        \frac{1}{Nd}\|X(\hat{\bbeta}-\bbeta^*)\|_2^2&\leq \frac{2}{Nd}W^{\top}X(\hat{\bbeta}-\bbeta^*) + \lambda_2\sum_{j=1}^d\Big(\|\bbeta^{*j}\|_2 - \|\hat{\bbeta}^j\|_2\Big)\\
        &\leq \frac{2}{Nd}\max_{j\in[d]}\big\|(X^{\top}W)^j\big\|_2 \sum_{j=1}^d\|\hat{\bbeta}^j - \bbeta^{*j}\|_2+ \lambda_2\sum_{j=1}^d\Big(\|\bbeta^{*j}\|_2 - \|\hat{\bbeta}^j\|_2\Big)
    \end{split}
\end{equation}
where the last inequality is from H$\ddot{\text{o}}$lder's inequality.

\textbf{Step 2.} Next, we will bound the noise term: $\|(X^{\top}W)^j\|_2$. From the definitions of $X$ and $W$, we write it explicitly as 
\begin{equation*}
    \|(X^{\top}W)^j\|_2 = \sqrt{\sum_{i=1}^d\Big(\sum_{n=1}^N\phi_j(x_n, a_n)\varepsilon_{ni}\Big)^2}.
\end{equation*}
It is easy to verify that $\{\phi_j(x_n, a_n)\varepsilon_{ni}\}_{n=1}^N$ is also a martingale difference sequence for any $i,j\in[d]$ and $|\phi_j(x_n, a_n)\varepsilon_{ni}|\leq 1$ since we assume $\|\phi(x,a)\|_{\infty}\leq 1$ for any state-action pair. According to Azuma-Hoeffding inequality (Lemma \ref{lemma:azuma}), for all $\tilde{\delta}>0$,
\begin{equation*}
    \begin{split}
        \mathbb P\Big(\Big|\sum_{n=1}^N\phi_j(x_n, a_n)\varepsilon_{ni}\Big|\geq \tilde{\delta}\Big)\leq 2\exp\Big(-\frac{\tilde{\delta}^2}{2N}\Big).
    \end{split}
\end{equation*}
Using the union bound twice, the following holds,
\begin{equation}\label{eqn:empirical_process}
     \mathbb P\Big(\max_{j\in[d]}\sqrt{\sum_{i=1}^d\Big(\sum_{n=1}^N\phi_j(x_n, a_n)\varepsilon_{ni}\Big)^2}\geq \sqrt{d}\tilde{\delta}\Big)\leq 2d^2\exp\Big(-\frac{\tilde{\delta}^2}{2N}\Big).
\end{equation}
Letting $\delta = 2d^2\exp(-\tilde{\delta}^2/2N)$, we have with probability at least $1-\delta$,
\begin{equation*}
    \frac{1}{Nd}\max_{j\in[d]}\big\|(X^{\top}W)^j\big\|_2\leq \frac{\lambda_2}{4}.
\end{equation*}
Define an event $\cA$ as 
\begin{equation*}
    \cA = \Big\{\frac{1}{Nd}\max_{j\in[d]}\big\|(X^{\top}W)^j\big\|_{2}\leq \frac{\lambda_2}{4}\Big\}.
\end{equation*}
Then we have $\mathbb P(\cA) \geq 1-\delta$.

\textbf{Step 3.} According to Karush–Kuhn–Tucker (KKT) condition, the solution $\hat{\bbeta}$ of the optimization problem Eq.~\eqref{eqn:group_lasso_rewrite} satisfies 
$$
\begin{cases}
\big(X^{\top}(Y-X\hat{\bbeta})\big)^j/(Nd) = \lambda_2 \hat{\bbeta}^j/\|\hat{\bbeta}^j\|_2,& \text{if} \  \hat{\bbeta}^j\neq 0,\\
\|\big(X^{\top}(Y-X\hat{\bbeta})\big)^j\|_2/(Nd) \leq \lambda_2, & \text{if} \ \hat{\bbeta}^j = 0.
\end{cases}$$
Under event $\cA$ and using KKT condition, we have if $\hat{\bbeta}^j\neq 0$, then
\begin{equation*}
    \begin{split}
        \lambda_2 = \big\|\frac{1}{Nd}\big(X^{\top}(Y-X\hat{\bbeta})\big)^j\big\|_2& = \big\|\frac{1}{Nd}\big(X^{\top}X(\bbeta^* -\hat{\bbeta})\big)^j +\frac{1}{Nd}\big(X^{\top}W\big)^j\big\|_2\\
        &\leq  \big\|\frac{1}{Nd}\big(X^{\top}X(\bbeta^* -\hat{\bbeta})\big)^j\big\|_2 +\big\|\frac{1}{Nd}\big(X^{\top}W\big)^j\big\|_2\\ 
        &\leq  \big\|\frac{1}{Nd}\big(X^{\top}X(\bbeta^* -\hat{\bbeta})\big)^j\big\|_2 + \frac{1}{4}\lambda_2,
    \end{split}
\end{equation*}
which implies
\begin{equation}\label{eqn:Bound1}
    \begin{split}
        \frac{1}{(Nd)^2}\big\|\big(X^{\top}X(\hat{\bbeta} - \bbeta^*)\big)_{\cS(\hat{\bbeta})}\big\|_2^2&=  \frac{1}{(Nd)^2}\sum_{j\in \cS(\hat{\bbeta})}\big\|\big(X^{\top}X(\hat{\bbeta} - \bbeta^*)\big)^j\big\|_2^2\\
        &\geq |\cS(\hat{\bbeta})|\frac{9}{16}\lambda_2^2.
    \end{split}
\end{equation}
We define a notation of  restricted maximum eigenvalue with respect to $\cS(\bbeta^*)$ and $X$:
\begin{equation}\label{def:rme}
    \tilde{C}_{\max}(m) = \max_{\bbeta\in\mathbb R^{d^2\times d^2}:\|\bbeta_{\cS(\bbeta^{*})^c}\|_{2,0}\leq m}\frac{\bbeta^{\top}X^{\top}X\bbeta}{N\|\bbeta\|_2^2}.
\end{equation}
Denote $\hat{m} = |\cS(\hat{\bbeta}) \setminus \cS(\bbeta^*)|$. Then we have 
\begin{equation}\label{eqn:Bound2}
    \begin{split}
        \big\|\big(X^{\top}X(\hat{\bbeta} - \bbeta^*)\big)_{\cS(\hat{\bbeta})}\big\|_2&\leq \sup_{\|\balpha_{\cS(\bbeta^*)^c}\|_{2,0}\leq \hat{m}}\frac{|\balpha^{\top}X^{\top}X(\hat{\bbeta} - \bbeta^*)|}{\|\balpha\|_2}\\
        &\leq \sup_{\|\balpha_{\cS(\bbeta^*)^c}\|_{2,0}\leq \hat{m}}\frac{\|\balpha^{\top}X^{\top}\|_2\|X(\hat{\bbeta} - \bbeta^*)\|_2}{\|\balpha\|_2}\\
        &=\sup_{\|\balpha_{\cS(\bbeta^*)^c}\|_{2,0}\leq \hat{m}}\frac{\sqrt{|\balpha^{\top}X^{\top}X\balpha|}}{\|\balpha\|_2}\|X(\hat{\bbeta} - \bbeta^*)\|_2\\
        &\leq \sqrt{N C_{\max}(\hat{m})} \|X(\hat{\bbeta} - \bbeta^*)\|_2.
    \end{split}
\end{equation}
Combining Eqs.~\eqref{eqn:Bound1} and \eqref{eqn:Bound2} together, we have 
\begin{equation}\label{eqn:group_bound1}
    |\cS(\hat{\bbeta})|\leq \frac{16C_{\max}(\hat{m})}{9Nd^2\lambda_2^2}\big\|X(\hat{\bbeta}-\bbeta^*)\big\|_2^2,
\end{equation}
holds with probability at least $1-\delta$.

\textbf{Step 4.} It remains to control the in-sample prediction error $\|X(\hat{\bbeta}-\bbeta^*)\|_2^2$. Under event $\cA$, using Eq.~\eqref{eqn:decom} implies
\begin{equation*}
    \begin{split}
        \frac{1}{Nd}\big\|X(\hat{\bbeta}-\bbeta^*)\big\|_2^2\leq \frac{\lambda_2}{2}\sum_{j=1}^d\big\|\hat{\bbeta}^j-\bbeta^{*j}\big\|_2 + \lambda_2\sum_{j=1}^d\Big(\|\bbeta^{*j}\|_2 - \|\hat{\bbeta}^j\|_2\Big).
    \end{split}
\end{equation*}
Adding $\sum_{j=1}^d\big\|\hat{\bbeta}^j-\bbeta^{*j}\big\|_2\lambda_2/2$ to both sides and using the fact that $\|\bbeta^{*j}-\bbeta^{*j}\|_2+ \|\hat{\bbeta}^j\|_2- \|\hat{\bbeta}^j\|_2=0$ for $j\neq \cS(\bbeta^*)$, we have 
\begin{equation}\label{eqn:decom2}
\begin{split}
    &\frac{1}{Nd}\big\|X(\hat{\bbeta}-\bbeta^*)\big\|_2^2 + \frac{\lambda_2}{2}\sum_{j=1}^d\big\|\hat{\bbeta}^j-\bbeta^{*j}\big\|_2 \\
    \leq&  \lambda_2\sum_{j\in \cS(\bbeta^*)}\Big(\big\|\hat{\bbeta}^j - \bbeta^{*j}\big\|_2 + \|\bbeta^{*j}\|_2 - \|\hat{\bbeta}^j\|_2\Big)\\
    \leq& 2\lambda_2 \sqrt{s} \big\|(\hat{\bbeta}-\bbeta^*)_{\cS(\bbeta^*)}\big\|_2,
\end{split}
\end{equation}
where the last inequality is from Cauchy-Schwarz inequality. Recall that the expected uncentered covariance matrix is defined as 
\begin{equation*}
      \Sigma=\mathbb E\Big[\frac{1}{L}\sum_{h=0}^{L-1}\phi(x_{h}^{(1)},a_{h}^{(1)})\phi(x_{h}^{(1)},a_{h}^{(1)})^{\top}\Big],
\end{equation*}
and we define the empirical uncentered covariance matrix as 
\begin{equation*}
    \begin{split}
        \hat{\Sigma} = \frac{1}{N}\sum_{n=1}^N\phi(x_n,a_n)\phi(x_n ,a_n)^{\top} = \frac{1}{K}\sum_{k=1}^K\frac{1}{L}\sum_{h=0}^{L-1}\phi(x_h^{(k)},a_h^{(k)})\phi(x_h^{(k)},a_h^{(k)})^{\top},
    \end{split}
\end{equation*}
with $N = KL$. Denote the expected and empirical uncentered covariance matrices for the multivariate linear regression as
  $$  
 \begin{matrix}
\hat{\Psi}=\begin{pmatrix}
\hat{\Sigma} & \ldots& 0 \\ 
\vdots&\ddots & \vdots\\
0&\ldots&\hat{\Sigma}
\end{pmatrix}\in\mathbb R^{d^2\times d^2};   \Psi= \begin{pmatrix}
\Sigma & \ldots& 0 \\ 
\vdots&\ddots & \vdots\\
0&\ldots&\Sigma
\end{pmatrix}\in\mathbb R^{d^2\times d^2}.
\end{matrix}
$$
We introduce a generalization of restricted eigenvalue condition (Definition \ref{def:RE}) for multivariate linear regression.
\begin{definition}[$\ell_{2,1}$-restricted eigenvalue]
Given a symmetric matrix $H\in\mathbb R^{d^2\times d^2}$ and integer $s\geq 1$, the restricted eigenvalue of $H$ is defined as
\begin{equation*}
    \tilde{C}_{\min}(H, s):=\min_{\cS\subset [d], |\cS|\leq s}\min_{\bbeta\in\mathbb R^{d^2}}\Big\{\frac{\langle \bbeta, H\bbeta \rangle}{\|\bbeta_{\cS}\|_2^2}: \bbeta\in\mathbb R^{d^2}, \|\bbeta_{\cS^c}\|_{2,1}\leq 3\|\bbeta_{\cS}\|_{2,1}\Big\}.
\end{equation*}
\end{definition}
Next lemma provides a lower bound for $\tilde{C}_{\min}(\hat{\Phi}, s)$. The proof is deferred to Appendix \ref{sec:proof_RE_condition}.
\begin{lemma}\label{lemma:RE_condition}
Assume the $\ell_{2,1}$-restricted eigenvalue of $\Psi$ satisfies $\tilde{C}_{\min}(\Psi,s)>0$ for some $\delta>0$. Suppose the sample size satisfies $N\geq 32^2 L\log (3d^2/\delta)s^2/\tilde{C}_{\min}(\Psi, s)^2$. Then the $\ell_{2,1}$-restricted eigenvalue of $\hat{\Psi}$ satisfies $\tilde{C}_{\min}(\hat{\Psi}, s)>\tilde{C}_{\min}(\Psi, s)/2$ with probability at least $1-\delta$.
\end{lemma}

On the other hand, from Eq.~\eqref{eqn:decom2}, we know that 
\begin{equation*}
\begin{split}
    \frac{1}{2}\sum_{j=1}^d\big\|\hat{\bbeta}^j-\bbeta^{*j}\big\|_2 &\leq 2\sum_{j\in \cS(\bbeta^*)}\big\|\hat{\bbeta}^j - \bbeta^{*j}\big\|_2,
\end{split}
\end{equation*}
and thus,
\begin{equation}\label{eqn:beta_restricted_set}
\begin{split}
    \sum_{j\in \cS(\bbeta^{*})^c}\big\|\hat{\bbeta}^j-\bbeta^{*j}\big\|_2 &\leq 3\sum_{j\in \cS(\bbeta^*)}\big\|\hat{\bbeta}^j - \bbeta^{*j}\big\|_2.
\end{split}
\end{equation}
This implies that $\|(\hat{\bbeta} - \bbeta^{*})_{\cS(\bbeta^{*})^c}\|_{2,1}\leq 3\|(\hat{\bbeta} - \bbeta^{*})_{\cS(\bbeta^{*})}\|_{2,1}$. Applying Lemma \ref{lemma:RE_condition}, the following holds with probability at least $1-\delta$,
\begin{equation}\label{eqn:re_bound}
\frac{\big\|X(\hat{\bbeta}-\bbeta^*)\big\|_2^2}{N\big\|(\hat{\bbeta}-\bbeta^*)_{\cS(\bbeta^*)}\big\|_2^2} = \frac{(\hat{\bbeta}-\bbeta^*)^{\top}\hat{\Psi}(\hat{\bbeta}-\bbeta^*)}{\big\|(\hat{\bbeta}-\bbeta^*)_{\cS(\bbeta^*)}\big\|_2^2}\geq \tilde{C}_{\min}(\hat{\Psi},s)\geq  \tilde{C}_{\min}(\Psi,s)/2.
\end{equation}
Plugging the above bound into Eq.~\eqref{eqn:decom2},
\begin{equation*}
\begin{split}
    \frac{1}{Nd}\big\|X(\hat{\bbeta}-\bbeta^*)\big\|_2^2&\leq 2\lambda_2\sqrt{s}\big\|(\hat{\bbeta}-\bbeta^*)_{\cS(\bbeta^*)}\big\|_2\\
    &\leq \frac{1}{\sqrt{N}}\big\|X(\hat{\bbeta}-\bbeta^*)\big\|_2\frac{4\sqrt{s}\lambda_2}{\sqrt{\tilde{C}_{\min}(\Psi,s)}}.
\end{split}
\end{equation*}
Combining with Eq.~\eqref{eqn:group_bound1} and the choice of $\lambda_2$ in Eq.~\eqref{eqn:lambda_2}, we reach
\begin{equation*}
    |\cS(\hat{\bbeta})|\leq  \frac{16\tilde{C}_{\max}}{9Nd^2\lambda_2^2} \frac{16\lambda_2^2Nd^2s}{\tilde{C}_{\min}(\Psi, s)} = \frac{256\tilde{C}_{\max}(\hat{m})s}{9\tilde{C}_{\min}(\Psi, s)},
\end{equation*}
with probability at least $1-\delta$, as long as $N\geq 32^2 L\log (3d^2/\delta)s^2/\tilde{C}_{\min}(\Psi, s)^2$.  If the vanilla restricted eigenvalue (Definition \ref{def:RE}) of $\Sigma$ satisfies $C_{\min}(\Sigma, s)>0$, then we have for any $\cS\subset [d], |\cS|\leq s$, and any $\bbeta_j\in\mathbb R^d$ satisfying $\|(\bbeta_j)_{\cS^c}\|_{1}\leq 3\|(\bbeta_j)_{\cS}\|_{1}$,
\begin{equation*}
    \frac{\bbeta_j^{\top}\Sigma \bbeta_{j}}{\|(\bbeta_j)_{\cS}\|_2^2} \geq C_{\min}(\Sigma, s)>0.
\end{equation*}
Consider a sequence of vectors $\bbeta_1, \ldots, \bbeta_d$ satisfying $\|(\bbeta_j)_{\cS^c}\|_{1}\leq 3\|(\bbeta_j)_{\cS}\|_{1}$. Then for $\bbeta = (\bbeta_1^{\top}, \ldots, \bbeta_d^{\top})^{\top}$, we have 
\begin{equation*}
  \bbeta^{\top}\Psi\bbeta =  \sum_{j=1}^d \bbeta_j^{\top}\Sigma \bbeta_{j} \geq C_{\min}(\Sigma,s)\sum_{j=1}^d \big\|(\bbeta_j)_{\cS}\big\|_2^2 =  C_{\min}(\Sigma,s)  \big\|\bbeta_{\cS}\big\|_2^2.
\end{equation*}
Therefore, we conclude $\tilde{C}_{\min}(\Psi,s)\geq C_{\min}(\Sigma, s)>0$ such that 
\begin{equation}\label{eqn:S_beta}
    |\cS(\hat{\bbeta})|\leq  \frac{256\tilde{C}_{\max}(\hat{m})s}{9C_{\min}(\Sigma, s)},
\end{equation}
with probability at least $1-\delta$, as long as $N\geq 32^2 L\log (3d^2/\delta)s^2/C_{\min}^2(\Sigma, s)$. 

 For any $\bbeta = (\bbeta_1^{\top}, \ldots, \bbeta_d^{\top})^{\top}$ satisfying $\|\bbeta_{\cS(\bbeta^*)^c}\|_{2,0}\leq m_0$, we have $\|(\bbeta_j)_{\cS(\bbeta^*)^c}\|_0\leq m_0$ for any $j\in[d]$. 
Given a positive semi-definite matrix $Z\in\mathbb R^{d\times d}$ and integer $s\geq 1$, the restricted maximum eigenvalue of $Z$ are defined as
\begin{equation*}
\begin{split}
 C_{\max}(Z, s):= \max_{\cS\subset [d], |\cS|\leq s}\max_{\bbeta\in\mathbb R^d}\Big\{\frac{\langle \bbeta, Z\bbeta \rangle}{\|\bbeta\|_2^2}: \bbeta\in\mathbb R^{d}, \|\bbeta_{\cS^c}\|_{0}\leq s\Big\}.
\end{split}
\end{equation*}
Using the definition of $C_{\max}(\hat{\Sigma}, s)$, it holds that
\begin{equation*}
    \bbeta_j^{\top}\hat{\Sigma}\bbeta_j\leq C_{\max}(\hat{\Sigma}, s) \|\bbeta_j\|_2^2, \ \text{for any} \ j\in[d]. 
\end{equation*}
Summing the above inequality from 1 to $d$, 
\begin{equation*}
  \bbeta^{\top} \frac{X^{\top}X}{N}\bbeta= \sum_{j=1}^d\bbeta_j^{\top}\hat{\Sigma}\bbeta_j\leq C_{\max}(\hat{\Sigma}, s) \sum_{j=1}^d\|\bbeta_j\|_2^2 = C_{\max}(\hat{\Sigma}, s)\|\bbeta\|_2^2.
\end{equation*}
This implies $\tilde{C}_{\max}(\hat{m})\leq C_{\max}(\hat{\Sigma}, \hat{m})$. As shown in the Lemma 1 in \cite{belloni2013least}, we have $C_{\max}(\hat{\Sigma}, m)\leq 4C_{\max}(\Sigma, m)$ for any $m+s\leq \log(n)$ as long as $n \gtrsim s$.

\textbf{Step 5.} Recall that $\hat{m} = |\cS(\hat{\bbeta})\setminus \cS(\bbeta^*)|$ and denote 
\begin{equation*}
   \cM = \Big\{m\in\mathbb N^{+}: m>\frac{2048sC_{\max}(\Sigma, m)}{9C_{\min}(\Sigma, s)}\Big\} .
\end{equation*}
Suppose there is a $m_0\in\cM$ such that $\hat{m}>m_0$. From Eq.~\eqref{eqn:S_beta}, we know that
\begin{equation*}
    \hat{m}\leq |\cS(\hat{\bbeta})|\leq C_{\max}(\Sigma, m) \frac{1024s}{9C_{\min}(\Sigma, s)}. 
\end{equation*}
According to Lemma 3 in \cite{belloni2013least} for the sublinearity of sparse maximum eigenvalues, we have 
\begin{equation*}
    C_{\max}(\Sigma, \hat{m}) \leq \lceil  \hat{m}/m_0\rceil   C_{\max}(\Sigma, m_0) \leq   2C_{\max}(\Sigma, m_0)\hat{m}/m_0.
\end{equation*}
where the last inequality we use $\lceil  \kappa\rceil \leq 2\kappa $. Putting the above two results together, we have 
\begin{equation*}
    m_0\leq \frac{2048s  C_{\max}(\Sigma, m_0)}{9C_{\min}(\Sigma, s)}.
\end{equation*}
This leads a contradiction with the definition of $\cM$. Therefore, $\hat{m}\leq m_0$ for all $m_0\in\cM$. This implies 
\begin{equation*}
\begin{split}
     |\cS(\hat{\bbeta})|&\leq  \min_{m_0\in\cM} C_{\max}(\Sigma, m_0) \frac{1024s}{9C_{\min}(\Sigma, s)}\\
     &= \Big[ \frac{1024\min_{m_0\in\cM}C_{\max}(\Sigma, m_0)}{9C_{\min}(\Sigma, s)}\Big] s\lesssim s.
\end{split}
\end{equation*}
The term $\min_{m_0\in\cM}C_{\max}(\Sigma, m_0)/C_{\min}(\Sigma, s)$ essentially characterizes the condition number of $\Sigma$ on a restricted support and is upper bounded by the condition number defined in the full support. Now we finish the proof of the first part of Lemma \ref{lemma:group_lasso} and start to prove the second part of Lemma \ref{lemma:group_lasso} under separability  condition.

According to Eq.~\eqref{eqn:beta_restricted_set}, under event $\cA$ we have 
\begin{equation*}
\begin{split}
     \big\|\hat{\bbeta}-\bbeta^*\big\|_{2,1} = \sum_{j=1}^d\big\|\hat{\bbeta}^j-\bbeta^{*j}\big\|_2 &= \sum_{j\in\cS(\bbeta^*)}\big\|\hat{\bbeta}^j-\bbeta^{*j}\big\|_2 + \sum_{j\in\cS(\bbeta^{*})^c}\big\|\hat{\bbeta}^j-\bbeta^{*j}\big\|_2\\
     & \leq4\sum_{j\in\cS(\bbeta^*)}\big\|\hat{\bbeta}^j-\bbeta^{*j}\big\|_2.
\end{split}
\end{equation*}
From Eq.~\eqref{eqn:re_bound},
\begin{equation*}
\begin{split}
     \sum_{j\in\cS(\bbeta^*)}\big\|\hat{\bbeta}^j-\bbeta^{*j}\big\|_2 &\leq \sqrt{s}\big\|(\hat{\bbeta}-\bbeta^*)_{\cS(\bbeta^*)}\big\|_2\\
     &\leq \sqrt{\frac{2s}{\tilde{C}_{\min}(\Sigma, s)}}\frac{1}{\sqrt{N}} \big\|X(\hat{\bbeta}-\bbeta^*)\big\|_2\\
     &\leq \frac{4\sqrt{2}sd\lambda_2}{\tilde{C}_{\min}(\Sigma, s)}\leq \frac{4\sqrt{2}sd\lambda_2}{C_{\min}(\Sigma, s)}.
\end{split}
\end{equation*}
Combining the above two inequality together and plugging in the choice of $\lambda_2$, we can bound
\begin{equation*}
   \big\|\hat{\bbeta}-\bbeta^*\big\|_{2,1}\leq \frac{64\sqrt{2}s\sqrt{d}}{C_{\min}(\Sigma,s)} \sqrt{\frac{2\log(2d^2/\delta)}{N}}.
\end{equation*}
with probability at least $1-\delta$. Under Assumption \ref{ass:signal}, the following holds that with probability at least $1-\delta$,
\begin{equation*}
   \min_{j\in \cS(\bbeta^*)} \big\|\bbeta^{*j}\big\|_2 > \big\|\hat{\bbeta} -\bbeta^*\big\|_{2,1}\geq \big\|\hat{\bbeta} -\bbeta^*\big\|_{2,\infty}.
\end{equation*}
If there is a $j\in \cS(\bbeta^*)$ but $j\notin\cS(\hat{\bbeta})$, we have 
\begin{equation*}
    \big\|\hat{\bbeta}^j-\bbeta^{*j}\big\|_2 = \big\|\bbeta^{*j}\big\|_2> \big\|\hat{\bbeta} -\bbeta^*\big\|_{2,\infty}.
\end{equation*}
On the other hand,
\begin{equation*}
     \big\|\hat{\bbeta}^j-\bbeta^{*j}\big\|_2\leq \big\|\hat{\bbeta} -\bbeta^*\big\|_{2,\infty},
\end{equation*}
which leads a contradiction. Now we conclude that $\cS(\hat{\bbeta})\supseteq \cK$. This ends the proof.  \hfill $\blacksquare$\\

\subsection{Proof of Theorem \ref{thm:upper_bound_eva2}: instance-dependent upper bound}\label{proof:upper_bound_eva2}

We restate the instance-dependent error bound error bound of vanilla fitted Q-evaluation algorithm on the full support.
\begin{theorem}[Theorem 5 in \cite{duan2020minimax}]\label{thm:instance_dependent_org}
 Assume the DMDPs satisfy Assumption \ref{assum:sparse_MDP} and batch dataset satisfy Assumption \ref{assum:sparse_MDP}. Suppose $\phi(x,a)^{\top}\Sigma^{-1}\phi(x,a)\lesssim d$ for any pair of $(x,a)$. Let $\delta\in(0,1)$ and Algorithm \ref{alg:evaluation_restricted} without feature selection stage takes $N$ samples satisfying 
	\begin{equation*}
	    N\gtrsim \frac{\gamma^2 \log(d/\delta)d}{(1-\gamma)^3}.
	\end{equation*}
	Set regularization parameter $\lambda_3 = \lambda_{\min}(\Sigma)\log(12d/\delta)C_1d/(1-\gamma)$. Letting the number of iteration $T\to\infty$, the following holds with probability at least $1-\delta$,
\begin{equation*}
\begin{split}
     \hat{v}^{\pi}_{\hat{w}}-v^{\pi} &\leq\frac{1}{1-\gamma} \sum_{t=0}^{\infty}\gamma^{t+1}\sqrt{(\nu_t^{\pi})^{\top}\Sigma^{-1} (\nu_t^{\pi})}\sqrt{\frac{\log (1/\delta)}{N}} + \frac{\gamma \ln(12d/\delta)d}{N(1-\gamma)^{3.5}},
\end{split}
\end{equation*}
where $\nu_t^{\pi} = \mathbb E^{\pi}[\phi(x_t,a_t)|x\sim\xi_0]$.
\end{theorem}

If the true relevant feature set $\cK$ is known in an oracle case, we could directly run the algorithm on $\cK$ such that all the dependency on $d$ can be reduced to $s$ and the instance-dependent term turns to be defined in the $\cK$ that is much sharper than the original one. Fortunately, Lemma \ref{lemma:group_lasso} implies $\hat{\cK}\supseteq\cK$ and $|\hat{\cK}|\lesssim s$. Suppose 
\begin{equation*}
     N\gtrsim \frac{\gamma^2 \log(s/\delta)s}{(1-\gamma)^3}\gtrsim \frac{\gamma^2 \log(|\hat{\cK}|/\delta)|\hat{\cK}|}{(1-\gamma)^3}.
\end{equation*}
Rewriting Theorem \ref{thm:instance_dependent_org} with respect to $\hat{\cK}$, we have 
\begin{equation*}
\begin{split}
     |\hat{v}^{\pi}_{\hat{w}}-v^{\pi}| \lesssim\frac{1}{1-\gamma} \sum_{t=0}^{\infty}\gamma^{t+1}\sqrt{(\tilde{\nu}_t^{\pi})^{\top}\tilde{\Sigma}^{-1} (\tilde{\nu}_t^{\pi})}\sqrt{\frac{\log (1/\delta)}{N}},
\end{split}
\end{equation*}
where $\tilde{\nu}_t^{\pi} = [\nu_t^{\pi}]_{\hat{\cK}}$ and $\tilde{\Sigma} = \Sigma_{\hat{\cK}\times \hat{\cK}}$. The corresponding condition $\phi(x,a)^{\top}\Sigma^{-1}\phi(x,a)\lesssim |\hat{\cK}|$ can be satisfied due to $C_{\min}(\Sigma, s)>0$ and $\|\phi(x,a)\|_{\infty}\leq 1$. From Definitions \ref{def:om}, \ref{def:chi-square} and Lemma B.2 in \cite{duan2020minimax}, we have
\begin{equation*}
\begin{split}
   \sqrt{1 + \chi^2_{\cG(\hat{\cK})}(\mu^{\pi}, \bar{\mu})} &=  \sum_{t=0}^{\infty}\gamma^t \sup_{f\in\cG(\hat{\cK})}\frac{ (1-\gamma)\mathbb E^{\pi}[f(x_t, a_t)|x_0\sim \xi_0]}{\sqrt{\mathbb E[\frac{1}{L}\sum_{h=0}^{L-1}f^2(x_{1h},a_{1h})]}} \\
    &=   (1-\gamma)\sum_{t=0}^{\infty}\gamma^{t}\sqrt{(\tilde{\nu}_t^{\pi})^{\top}\tilde{\Sigma}^{-1} (\tilde{\nu}_t^{\pi})}.
\end{split}
\end{equation*}
Now we end the proof. \hfill $\blacksquare$\\

\section{Proof of Theorem \ref{thm:agm-result}: lasso fitted Q-iteration}\label{proof:upper_bound}

\subsection{Proof of Theorem \ref{thm:agm-result}}

The main structure of this proof is similar to the proof of Theorem \ref{thm:upper_bound_eva1} in Appendix \ref{proof:upper_bound_eva1} but we need to utilize the contraction property of Bellman optimality operator. Recall that we split the whole dataset into $T$ folds and each fold consists of $R$ episodes or $RL$ sample transitions. The overall sample size is $N=TRL$.

\textbf{Step 1.}  We verify that the execution of Algorithm \ref{alg:batch_opt} is equivalent to the approximate value iteration.  Recall that a generic Lasso estimator with respect to a function $V$ at $t$th phase is defined in Eq.~\eqref{eqn:sparse_solution} as 
\begin{equation*}
\hat{w}_{t}(V) = \argmin_{w\in\mathbb R^d}\Big(\frac{1}{RL}\sum_{i=1}^{RL}\Big(\Pi_{[0,1/(1-\gamma)]}V(x_i^{(t)'}) -\phi(x_i^{(t)},a_i^{(t)})^{\top}w\Big)^2 + \lambda_1 \|w\|_1\Big).
\end{equation*}
Define $V_w(x)=\max_{a\in\cA}(r(x,a)+\gamma \phi(x,a)^{\top}w)$. For simplicity, we write $\hat{w}_t := \hat{w}_{t}(V_{\hat{w}_{t-1}})$ for short. Define an approximate Bellman optimality operator $\hat{\cT}^{(t)}:\mathbb R^{\cX}\to\mathbb R^{\cX}$ as:
\begin{equation}\label{eqn:T1}
    [\hat{\cT}^{(t)}V](x) := \max_a\Big[r(x,a) + \gamma\phi(x,a)^{\top}\hat{w}_t(V)\Big].
\end{equation}
Note this $\hat{\cT}^{(t)}$ is a randomized operator that only depends data collected at $t$th phase.
 Algorithm \ref{alg:batch_opt} is equivalent to the following approximate value iteration:
\begin{equation}\label{eqn:api}
    [\hat{\cT}^{(t)}\Pi_{[0,1/(1-\gamma)]}V_{\hat{w}_{t-1}}](x) = \max_a\Big[r(x,a) + \gamma\phi(x,a)^{\top}\hat{w}_t\Big] = \max_a Q_{\hat{w}_t}(x,a) = V_{\hat{w}_t}(x).
\end{equation}
Recall that the true Bellman optimality operator $\cT:\mathbb R^{\cX}\to\mathbb R^{\cX}$ is defined as 
\begin{equation}\label{eqn:true_bell}
    [\cT V](x) := \max_a\Big[r(x,a)+\gamma \sum_{x'}P(x'|x,a)V(x')'\Big].
\end{equation}

\textbf{Step 2.} We verify that the true Bellman operator on $\Pi_{[0,1/(1-\gamma)]}V_{\hat{w}_{t-1}}$ can also be written as a linear form. From Condition \ref{assum:sparse_MDP}, there exists some functions $\psi(\cdot) = (\psi_k(\cdot))_{k\in\cK}$ such that  for every $x,a,x'$, the transition function can be represented as 
\begin{equation}\label{eqn:linear_MDP}
    P(x' |x,a) = \sum_{k\in \cK} \phi_k(x,a) \psi_k(x'),
\end{equation}
where $\cK\subseteq [d]$ and $|\cK|\leq s$. For a vector $\bar{w}_t\in\mathbb R^d$, we define its $k$th coordinate as
\begin{equation}\label{eqn:def_w_bar}
    \bar{w}_{t,k} = \sum_{x'} \Pi_{[0,1/(1-\gamma)]} V_{\hat{w}_{t-1}}(x')\psi_k(x'), \ \text{if} \ k\in\cK,
\end{equation}
and $\bar{w}_{t,k}=0$ if $k\notin \cK$. By the definition of true Bellman optimality operator in Eq.~\eqref{eqn:true_bell} and Eq.~\eqref{eqn:linear_MDP},
\begin{eqnarray}\label{eqn:true_bell_linear}
     [\cT \Pi_{[0,1/(1-\gamma)]}V_{\hat{w}_{t-1}}](x) &=& \max_a\Big[r(x,a)+\gamma \sum_{x'}P(x'|x,a)\Pi_{[0,1/(1-\gamma)]}V_{\hat{w}_{t-1}}(x')'\Big]\nonumber\\
     &=& \max_a\Big[r(x,a)+\gamma \sum_{x'}\phi(x,a)^{\top}\psi(x')\Pi_{[0,1/(1-\gamma)]}V_{\hat{w}_{t-1}}(x')'\Big]\nonumber\\
     &=& \max_a\Big[r(x,a)+\gamma \sum_{x'}\sum_{k\in\cK}\phi_k(x,a)\psi_k(x')\Pi_{[0,1/(1-\gamma)]}V_{\hat{w}_{t-1}}(x')'\Big]
    \nonumber\\
    &=& \max_a\Big[r(x,a)+\gamma \sum_{k\in\cK}\phi_k(x,a)\sum_{x'}\psi_k(x')\Pi_{[0,1/(1-\gamma)]}V_{\hat{w}_{t-1}}(x')'\Big]
    \nonumber\\
      &=& \max_a\Big[r(x,a)+\gamma\phi(x,a)^{\top}\bar{w}_t\Big].
\end{eqnarray}

\textbf{Step 3.} We start to bound $\|V_{\hat{w}_{t}}-v^*\|_{\infty}$ for each phase $t$. By the approximate value iteration form Eq.~\eqref{eqn:api} and the definition of optimal value function,
\begin{equation}\label{eqn:V_decomp}
    \begin{split}
         \big\|V_{\hat{w}_{t}}-v^*\big\|_{\infty}&= \big\|\hat{\cT}^{(t)}\Pi_{[0,1/(1-\gamma)]}V_{\hat{w}_{t-1}} - \cT v^*\big\|_{\infty} \\
   &=\big\|\hat{\cT}^{(t)}\Pi_{[0,1/(1-\gamma)]}V_{\hat{w}_{t-1}} - \cT\Pi_{[0,1/(1-\gamma)]}V_{\hat{w}_{t-1}} \big\|_{\infty} \\
   &+ \big\|\cT\Pi_{[0,1/(1-\gamma)]}V_{\hat{w}_{t-1}} - \cT v^* \big\|_{\infty}.
    \end{split}
\end{equation}
The first term mainly captures the error between approximate Bellman optimality operator and true Bellman optimality operator while the second term can be bounded by the contraction of true Bellman operator. From linear forms Eqs.~\eqref{eqn:api} and \eqref{eqn:true_bell_linear}, it holds for any $x\in\cX$,
\begin{eqnarray}\label{eqn:approx_Bellman}
   &&[\hat{\cT}^{(t)}\Pi_{[0,1/(1-\gamma)]}V_{\hat{w}_{t-1}}](x) - [\cT\Pi_{[0,1/(1-\gamma)]}V_{\hat{w}_{t-1}}](x)\nonumber \\
   &=& \max_a\Big[r(x,a) + \gamma\phi(x,a)^{\top}\hat{w}_t\Big] - \max_a\Big[r(x,a)+\gamma \phi(x,a)^{\top}\bar{w}_t\Big]\nonumber\\
   &\leq& \gamma \max_a\big|\phi(x,a)^{\top}(\hat{w}_t-\bar{w}_t)\big|\nonumber\\
   &\leq& \gamma \max_{a,x}\|\phi(x,a)\|_{\infty}\|\hat{w}_t-\bar{w}_t\|_1.
\end{eqnarray}
Applying Lemma \ref{lemma:lasso_l1_bound}, with the choice of $\lambda_1=(1-\gamma)^{-1}\sqrt{\log (2d/\delta)/RL}$, the following error bound holds with probability at least $1-\delta$, 
\begin{equation}\label{eqn:lasso_error_bound2}
  \big\|\hat{w}_t-\bar{w}_t\big\|_1\leq \frac{16\sqrt{2}s}{C_{\min}(\Sigma, s)}\frac{1}{1-\gamma}\sqrt{\frac{\log(2d/\delta)}{RL}},
\end{equation}
where $R$ satisfies $ R\geq C_1\log(3d^2/\delta)s^2/C_{\min}(\Sigma, s).$

Note that the samples we use between phases are mutually independent. Thus Eq.~\eqref{eqn:lasso_error_bound2} uniformly holds for all $t\in[T]$ with probability at least $1-T\delta$. Plugging it into Eq.~\eqref{eqn:approx_Bellman}, we have for any phase $t\in[T]$,
\begin{eqnarray}\label{eqn:bound_1}
    \big\|\hat{\cT}^{(t)}\Pi_{[0,1/(1-\gamma)]}V_{\hat{w}_{t-1}} - \cT\Pi_{[0,1/(1-\gamma)]}V_{\hat{w}_{t-1}} \big\|_{\infty}\leq \gamma \frac{16\sqrt{2}s}{C_{\min}(\Sigma, s)}\frac{1}{1-\gamma}\sqrt{\frac{\log(2dT/\delta)}{RL}},
\end{eqnarray}
holds with probability at least $1-\delta$. 

To bound the second term in Eq.~\eqref{eqn:V_decomp}, we use the contraction property of true Bellman operator such that
\begin{equation}\label{eqn:bound_2}
    \big\|\cT\Pi_{[0,1/(1-\gamma)]}V_{\hat{w}_{t-1}} - \cT v^* \big\|_{\infty} \leq  \gamma\big\|\Pi_{[0,1/(1-\gamma)]}V_{\hat{w}_{t-1}} -  v^* \big\|_{\infty}.
\end{equation}
Plugging Eqs.~\eqref{eqn:bound_1} and \eqref{eqn:bound_2} into Eq.~\eqref{eqn:V_decomp}, it holds that
\begin{equation}\label{eqn:V_bound}
    \big\|V_{\hat{w}_{t}}-v^*\big\|_{\infty}\leq \gamma \frac{16\sqrt{2}s}{C_{\min}(\Sigma, s)}\frac{1}{1-\gamma}\sqrt{\frac{\log(2dT/\delta)}{RL}}+\gamma\big\|\Pi_{[0,1/(1-\gamma)]}V_{\hat{w}_{t-1}} -  v^* \big\|_{\infty},
\end{equation}
with probability at least $1-\delta$. Recursively using Eq.~\eqref{eqn:V_bound}, the following holds with probability $1-\delta$,
\begin{equation*}
\begin{split}
         \big\|\Pi_{[0,1/(1-\gamma)]}V_{\hat{w}_{T-1}}- v^*\big\|_{\infty}
         \leq \big\|V_{\hat{w}_{T-1}}- v^*\big\|_{\infty}\\
    &= \gamma \frac{16\sqrt{2}s}{C_{\min}(\Sigma, s)}\frac{1}{1-\gamma}\sqrt{\frac{\log(2dT/\delta)}{RL}} + \gamma  \big\|\Pi_{[0,1/(1-\gamma)]}V_{\hat{w}_{T-2}} -  v^* \big\|_{\infty}\\
    &\leq \gamma^{T-1} \big\|\Pi_{[0,1/(1-\gamma)]}V_{\hat{w}_{0}} -  v^* \big\|_{\infty} + \sum_{t=1}^{T-1}\gamma^t \frac{16\sqrt{2}s}{C_{\min}(\Sigma, s)}\frac{1}{1-\gamma}\sqrt{\frac{\log(2dT/\delta)}{RL}}\\
    &\leq \frac{2\gamma^{T-1}}{1-\gamma} +\frac{1}{(1-\gamma)^2} \frac{16\sqrt{2}s}{C_{\min}(\Sigma, s)}\sqrt{\frac{\log(2dT/\delta)}{RL}},
    \end{split}
\end{equation*}
where the first inequality is due to that $\Pi_{[0,1/(1-\gamma)]}$ can only make error smaller and the last inequality is from $\sum_{t=1}^{T-1} \gamma^t\leq 1/(1-\gamma)$. By properly choosing $T = \Theta(\log (N/(1-\gamma))/(1-\gamma))$, it implies
\begin{eqnarray*}
    \big\|\Pi_{[0,1/(1-\gamma)]}V_{\hat{w}_{T-1}}- v^*\big\|_{\infty}\leq \frac{1}{(1-\gamma)^{5/2}}\frac{32\sqrt{2}s}{C_{\min}(\Sigma, s)}\sqrt{\frac{\log(2dT/\delta) \log (N/(1-\gamma))}{N}},
\end{eqnarray*}
holds with probability at least $1-\delta$. From Proposition 2.14 in \cite{bertsekas1995dynamic}, \begin{equation}\label{eqn:upper_bound}
     \big\|v^{\hat{\pi}_{T}} - v^*\big\|_{\infty} \leq \frac{1}{1-\gamma}\big\|Q_{\hat{w}_T}-Q^*\big\|_{\infty} \leq \frac{2}{1-\gamma}\big\|\Pi_{[0,1/(1-\gamma)]}V_{\hat{w}_{T-1}}-v^*\big\|_{\infty}. 
\end{equation}
Putting the above together, we have with probability at least $1-\delta$,
\begin{eqnarray*}
    \big\|v^{\hat{\pi}_{T}}-v^*\big\|_{\infty}\leq \frac{64\sqrt{2}s}{C_{\min}(\Sigma, s)}\sqrt{\frac{\log(2dT/\delta) \log (N/(1-\gamma))}{N(1-\gamma)^7}},
\end{eqnarray*}
when the sample size $N$ satisfies
\begin{equation*}
    N\geq \frac{C_1s^2L\log(3d^2/\delta)T}{C_{\min}(\Sigma, s)},
\end{equation*}
for some sufficiently large constant $C_1$. This ends the proof.  \hfill $\blacksquare$\\

\input{lowerbound.tex}


\section{Proofs of auxiliary results}\label{sec:auxiliary}
\subsection{Proof of Lemma \ref{lemma:RE_condition}}\label{sec:proof_RE_condition}
We prove if the population covariance matrix 
satisfies the restricted eigenvalue condition, the empirical covariance matrix satisfies it as well with high probability. Recall that
\begin{equation*}
    \begin{split}
        \hat{\Sigma} = \frac{1}{K}\sum_{k=1}^K\frac{1}{L}\sum_{h=0}^{L-1}\phi(x_h^{(k)},a_h^{(k)})\phi(x_h^{(k)},a_h^{(k)})^{\top},
    \end{split}
\end{equation*}
and
  $$  
 \begin{matrix}
\hat{\Psi}=\begin{pmatrix}
\hat{\Sigma} & \ldots& 0 \\ 
\vdots&\ddots & \vdots\\
0&\ldots&\hat{\Sigma}.
\end{pmatrix}.
\end{matrix}
$$
For any $i,j\in[d]$, define
\begin{equation*}
    v_{ij}^{(k)} = \frac{1}{L}\sum_{h=0}^{L-1}\phi_i(x_h^{(k)}, a_h^{(k)})\phi_j(x_h^{(k)}, a_h^{(k)})-\Sigma_{ij}.
\end{equation*}
It is easy to verify $\mathbb E[v_{ij}^{(k)}] = 0$ and $|v_{ij}^{(k)}|\leq 1$ since we assume $\|\phi(x,a)\|_{\infty}\leq 1$. Note that from the data collection process Assumption \ref{assm:data_colle}, samples between different episodes are independent. This implies $v_{ij}^{(1)}, \ldots,v_{ij}^{(K)}$ are independent. By standard Hoeffding's inequality (Proposition 5.10 in \cite{vershynin2010introduction}), we have 
\begin{equation*}
    \mathbb P\Big(\Big|\sum_{k=1}^Kv_{ij}^{(k)}\Big|\geq \delta\Big)\leq 3\exp\Big(-\frac{C_0\delta^2}{K}\Big),
\end{equation*}
for some absolute constant $C_0>0$. Applying the union bound over $i,j\in[d]$, we have
\begin{equation*}
\begin{split}
    &\mathbb P\Big(\max_{i,j}\Big|\sum_{k=1}^Kv_{ij}^{(k)}\Big|\geq \delta\Big)\leq 3d^2\exp\Big(-\frac{C_0\delta^2}{K}\Big)\\
    &\Rightarrow \mathbb P\Big(\big\|\hat{\Sigma} -\Sigma\big\|_{\infty}\geq \delta\Big)\leq 3d^2\exp\Big(-\frac{C_0\delta^2}{K}\Big).
\end{split}
\end{equation*}
Since the blocks of $\Psi$ are the same, the following holds holds with probability $1-\delta$.
\begin{equation*}
    \big\|\hat{\Psi} - \Psi\big\|_{\infty}\leq \sqrt{\frac{\log (3d^2/\delta)}{K}}.
\end{equation*}
Therefore, when the number of episodes $K\geq 32^2\log (3d^2/\delta)s^2/\tilde{C}_{\min}(\Psi, s)^2$, the following holds with probability at least $1-\delta$,
\begin{equation*}
     \big\|\hat{\Psi} - \Psi\big\|_{\infty}\leq \frac{\tilde{C}_{\min}(\Psi, s)}{32s}.
\end{equation*}
Next lemma shows that if the restricted eigenvalue condition holds for one positive semi-definite block diagonal matrix $\Sigma_0$, then it holds with high probability for another positive semi-definite block diagonal matrix $\Sigma_1$ as long as $\Sigma_0$ and $\Sigma_1$ are close enough in terms of entry-wise max norm.
\begin{lemma}[Corollary 6.8 in \citep{buhlmann2011statistics}]\label{lemma:eigen_concentration}
Let $\Sigma_0$ and $\Sigma_1$ be 
 two positive semi-definite block diagonal matrices. 
 Suppose that the restricted eigenvalue of $\Sigma_0$ satisfies $\tilde{C}_{\min}(\Sigma_0, s)>0$ and $\|\Sigma_1-\Sigma_0\|_{\infty}\leq \tilde{C}_{\min}(\Sigma_0, s)/(32s)$. 
 Then the restricted eigenvalue of $\Sigma_1$  satisfies $\tilde{C}_{\min}(\Sigma_1, s)>\tilde{C}_{\min}(\Sigma_0, s)/2$.
\end{lemma}
Applying Lemma \ref{lemma:eigen_concentration} with $\hat{\Psi}$ and $\Psi$, we have the restricted eigenvalue of $\hat{\Phi}$ satisfies $\tilde{C}_{\min}(\hat{\Psi}, s)>\tilde{C}_{\min}(\Psi, s)/2$ with probability at least $1-\delta$, as long as  the sample size $N\geq 32^2 L\log (3d^2/\delta)s^2/\tilde{C}_{\min}(\Psi, s)^2$. This ends the proof.  \hfill $\blacksquare$\\

\subsection{Proof of Lemma \ref{lemma:lasso_l1_bound}}\label{sec:proof_lasso_l1}
We prove the $\ell_1$-norm bound of estimating $\bar{w}_t$ using a fresh fold of batch data. We overload the notation $\hat{\Sigma}$ to denote
\begin{equation*}
    \hat{\Sigma} =\frac{1}{RL} \sum_{i=1}^{RL}\phi(x_i, a_i)\phi(x_i, a_i)^{\top} = \frac{1}{R}\sum_{r=1}^R \frac{1}{L}\sum_{h=1}^L \phi(x_h^{(r)}, a_h^{(r)})\phi(x_h^{(r)}, a_h^{(r)})^{\top} .
\end{equation*}
Similar to the proof of Lemma \ref{lemma:RE_condition} in Appendix \ref{sec:proof_RE_condition}, we can have with probability at least $1-\delta$,
\begin{equation*}
    \big\|\hat{\Sigma} -\Sigma\big\|_{\infty}\leq \sqrt{\frac{C_1}{R}\log\Big(\frac{3d^2}{\delta}\Big)},
\end{equation*}
 where $C_1$ is an absolute constant. When $R\geq C_132^2\log(3d^2/\delta)s^2/C_{\min}(\Sigma, s)$, we have 
 \begin{equation*}
      \big\|\hat{\Sigma} -\Sigma\big\|_{\infty}\leq \frac{C_{\min}(\Sigma, s)}{32s}.
 \end{equation*}
Applying Lemma \ref{lemma:eigen_concentration}, we have  $C_{\min}(\hat{\Sigma}, s)>C_{\min}(\Sigma, s)/2$ with probability at least $1-\delta$. Note that $\{\varepsilon_i\phi_j(x_i, a_i)\}_{i=1}^{RL}$ is a martingale difference sequence and $|\varepsilon_i\phi_j(x_i, a_i)|\leq 1/(1-\gamma)$.  Similar to the proof of Eq.~\eqref{eqn:empirical_process} by Azuma-Hoeffding inequality,
\begin{equation*}
    \mathbb P\Big(\max_{j\in[d]}\Big|\frac{1}{RL}\sum_{i=1}^{RL}\varepsilon_i\phi_{j}(x_i, a_i)\Big|\leq \frac{1}{1-\gamma}\sqrt{\frac{\log (2d/\delta)}{RL}}\Big)\geq 1-\delta.
\end{equation*}
Denote event $\cE$ as 
\begin{equation*}
    \cE = \Big\{\max_{j\in[d]}\Big|\frac{1}{RL}\sum_{i=1}^{RL}\varepsilon_i\phi_{j}(x_i, a_i)\Big|\leq \lambda_1\Big\}.
\end{equation*}
Then $\mathbb P(\cE)\geq 1-\delta$. Under event $\cE$, applying (B.31) in \cite{bickel2009simultaneous}, we have
\begin{equation*}
    \big\|\hat{w}_t-\bar{w}_t\big\|_1\leq \frac{16\sqrt{2}s\lambda_1}{C_{\min}(\Sigma, s)},
\end{equation*}
holds with probability at least $1-2\delta$. This ends the proof.  \hfill $\blacksquare$\\

\input{lower_bound_supp}

\section{Supporting lemmas}\label{sec:supporting}

\begin{lemma}
Let $Z_1,\ldots, Z_n$ be random, positive-semidefinite adaptively chosen matrices with dimension $d$. Suppose $\lambda_{\max}(Z_i)\leq R^2$ almost surely for all $i$. Let $Z^+ = \sum_{i=1}^nZ_i$ and $W = \sum_{i=1}^n \mathbb E[Z_i|Z_1,\ldots, Z_{i-1}]$. Then for any $\mu$ and any $\alpha\in(0, 1)$,
\begin{equation*}
    \mathbb P\Big(\lambda_{\min}(Z^+)\leq (1-\alpha)\mu \ \text{and} \ \lambda_{\min}(W)\geq \mu\Big)d\Big(\frac{1}{e^{\alpha}(1-\alpha)^{1-\alpha}}\Big)^{\mu/R^2}
\end{equation*}
\end{lemma}

\begin{lemma}[Azuma-Hoeffding's inequality]\label{lemma:azuma}
Let $\cF_n = \sigma(x_1,\ldots, x_n)$ be a sequence of $\sigma$-fields known as a filtration. Let $\{(x_n, \cF_n)\}_{n=1}^{\infty}$ be a martingale difference sequence for which there are constants $\{(a_k,b_k)_{k=1}^n\}$ such that $x_k\in[a_k, b_k]$ almost surely for $k=1,\ldots, n$. Then for all $t\geq 0$,
\begin{equation*}
    \mathbb P\Big(\Big|\sum_{k=1}^nx_k\Big|\geq t\Big)\leq 2\exp\Big(-\frac{2t^2}{\sum_{k=1}^n(b_k-a_k)^2}\Big).
\end{equation*}
\end{lemma}

\section{Preliminary experiments}\label{sec:exper}
 The left panel in Figure \ref{fig:mountain_car} shows that our Lasso-FQE clearly has smaller estimation error compared with FQE, proving the sparse feature selection is effective in a practical RL example.  The right panel in Figure \ref{fig:mountain_car} demonstrates how the distribution mismatch ($\chi^2$-divergence term) affects OPE error (with sample size fixed).  The results confirm our theorems that the (restricted) chi-square divergence sharply determines the (sparse) OPE error.

\begin{figure}[h]\label{fig:mountain_car}
 \centering
    \centering
    \includegraphics[width=0.35\textwidth]{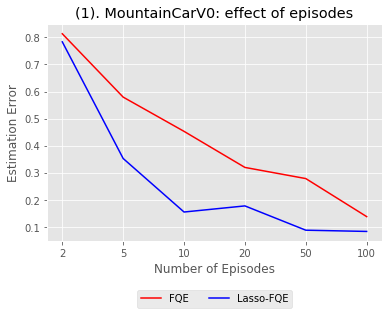}
    \includegraphics[width=0.35\textwidth]{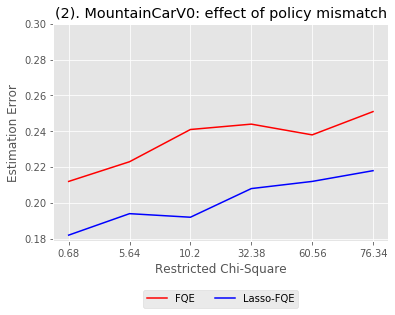}
    \vspace{-0.17in}
\end{figure}

\end{document}

%% file: lowerbound.tex
 \section{Proof of Theorem \ref{thm:lowerbound}: minimax lower bound of policy optimization}\label{sec:proof_lower_bound_BPO}
We first restate the full statement of Theorem \ref{thm:lowerbound} and include an instance-dependent lower bound. Note that the worse-case lower bound Eq.~\eqref{eqn:worse_PO_lower} can be derived from instance-dependent lower bound Eq.~\eqref{eqn:instance_PO_lower} (See Appendix \ref{sec:worse_PO_lower} for details). 

\paragraph{Full statement of Theorem \ref{thm:lowerbound}}
Suppose Assumption \ref{assm:data_colle} holds and $\gamma \geq \frac{2}{3}$. Let $\widehat{\pi}$ denote an algorithm that maps dataset $\mathcal{D}$ to a policy $\widehat{\pi}(\mathcal{D})$. If $N \gtrsim sL(1-\gamma)^{-1}$, then for any $\widehat{\pi}$, there always exists an DMDP instance $M \in \mathcal{M}_{\phi,s}(\mathcal{X}, \mathcal{A}, \gamma)$ with feature $\phi \in (\mathbb{R}^d)^{\mathcal{X} \times \mathcal{A}}$ satisfying $\|\phi(x,a)\|_{\infty} \leq 1$ for all $(x,a) \in \mathcal{X} \times \mathcal{A}$,
 then 
\begin{equation}\label{eqn:instance_PO_lower}
     \mathbb{P}_{M} \left( v_{\xi_0}^{*} - v_{\xi_0}^{\widehat{\pi}(\mathcal{D})}  \gtrsim\sqrt{1+\chi_{\mathcal{G}(\mathcal{K})}^2(\mu^{*},\bar{\mu})} \sqrt{\frac{1}{N(1-\gamma)^3}} \right) \geq \frac{1}{6}, 
\end{equation}
where $\mu^{*}$ is the discounted state-action occupancy measure of $\pi^*$.
In addition, we have
\begin{equation}\label{eqn:worse_PO_lower}
	\mathbb{P}_{M} \left( v_{\xi_0}^{*} - v_{\xi_0}^{\widehat{\pi}(\mathcal{D})} \gtrsim \sqrt{\frac{s }{C_{\min}(\Sigma,s)}} \sqrt{\frac{1}{N(1-\gamma)^3}} \right) \geq \frac{1}{6}.
\end{equation}

 
 \paragraph{Remark} Minimax sample complexity lower bound for solving MDP has been studied in the setting with a generative model that allows querying any $(s,a)$ for independent samples. \citet{azar2013minimax} constructed a hard instance of tabular MDP and, by reducing policy optimization to a testing a Bernoulli distribution, proved a lower bound $\frac{SA}{(1-\gamma)^3}$ which is known to be sharp. \citet{yang2019sample} extended the construction to linear MDP and show that the sample complexity lower bound is $\frac{d}{(1-\gamma)^3}$ under a generative model. There also exists matching upper bound in the same setting. 
 
 Our Theorem \ref{thm:lowerbound} applies to the setting of batch episodic data where are highly dependent. Due to this major difference, we have to use a more intricate proof based on likelihood test to establish a minimax lower bound. Further, Theorem \ref{thm:lowerbound} characterizes for the first time that the lower bound depends on the minimal eigenvalue of the data's population covariance.

 \subsection{Reducing to likelihood test}
 We prove the minimax lower bound by conducting likelihood test. Similar to Lemma C.1 in \cite{duan2020minimax}, we have Lemma \ref{lemma:LikeTest} below.
 
 \begin{lemma} \label{lemma:LikeTest}
 	Let $M_\alpha$ and $M_\beta$ be two MDP instances with transition kernels $p_\alpha(x' \, | \, x,a)$ and $p_\beta(x' \, | \, x,a)$. Suppose Assumption \ref{assm:data_colle} holds.
 	Define likelihood functions $$\mathcal{L}_i (\mathcal{D}) := \prod_{k=1}^{K} \bar{\xi}_0(x_0^{(k)}) \prod_{h=0}^{L-1} \bar{\pi}(a_h^{(k)} \, | \, x_h^{(k)}) p_i(x_{h+1}^{(k)} \, | \, x_h^{(k)}, a_h^{(k)}), \qquad i = \alpha, \beta. $$ 
 	Denote $\mathbb{P}_\alpha$ the probability space generated by running $M_{\alpha}$ following the behavioral policy $\bar\pi$.
 	If $\mathbb{P}_\alpha \Big( \frac{\mathcal{L}_\beta(\mathcal{D})}{\mathcal{L}_\alpha(\mathcal{D})} \geq \frac{1}{2} \Big) \geq \frac{1}{2}$ and there exist scalars $\rho_\alpha, \rho_\beta \geq 0$ such that
 	\begin{equation} \label{HardInstance} \big\{ \text{policy }\pi \, \big| \, v_{M_\alpha, \xi_0}^{*} - v_{M_\alpha, \xi_0}^{\pi} \geq \rho_\alpha \big\} \cap \big\{ \text{policy }\pi \, \big| \, v_{M_\beta, \xi_0}^{*} - v_{M_\beta, \xi_0}^{\pi} \geq \rho_\beta \big\} = \emptyset, \end{equation}
 	then for any policy learning algorithm $\widehat{\pi}$,
 	\begin{equation} \label{LikeTest} \mathbb{P}_\alpha \Big( v_{M_\alpha, \xi_0}^{*} - v_{M_\alpha, \xi_0}^{\widehat{\pi}(\mathcal{D})} \geq \rho_\alpha \Big) \geq \frac{1}{6} \qquad \text{or} \qquad \mathbb{P}_\beta \Big( v_{M_\beta, \xi_0}^{*} - v_{M_\beta, \xi_0}^{\widehat{\pi}(\mathcal{D})} \geq \rho_\beta \Big) \geq \frac{1}{6}. \end{equation}
 \end{lemma}
 
 We learn from Lemma \ref{lemma:LikeTest} that as long as $\mathbb{P}_\alpha \Big( \frac{\mathcal{L}_\beta(\mathcal{D})}{\mathcal{L}_\alpha(\mathcal{D})} \geq \frac{1}{2} \Big) \geq \frac{1}{2}$ and \eqref{HardInstance} hold, the lower bound is achieved at model $M_\alpha$ or $M_\beta$. In the following, we construct MDP models and analyze these two conditions separately.

\subsection{Constructing MDP instances}
We assume without loss of generality that the number of active features $s$ is even. We consider a simplest case where the MDP only consists of two states, {\it i.e.} $\mathcal{X} = \{ \overline{x}, \underline{x} \}$. At each state, the agent chooses from $\frac{s}{2} + s(d-s)$ actions $\mathcal{A} = \big\{ a_1, a_2, \ldots, a_{\frac{s}{2}} \big\} \cup \big\{ \bar{a}_{i,k} \, \big| \, i = 1,2,\ldots,\frac{s}{2}, \, k = \pm 1, \pm 2, \ldots, \pm(d-s) \big\}$. Here, we only use $\bar{a}_{i,j}$ in collecting the dataset $\mathcal{D}$.

We first introduce Lemma \ref{lemma:DCT}, which will be used in the construction of feature mapping $\phi: \mathcal{X} \times \mathcal{A} \rightarrow \mathbb{R}^d$.
\begin{lemma} \label{lemma:DCT}
	For any $s \in \mathbb{Z}_+$, there exists an $s$-by-$s$ orthogonal matrix $\Theta \in \mathbb{R}^{s \times s}$
	satisfying
	\begin{equation} \label{l_infty} \sqrt{s} \cdot |\Theta_{i,j}| \leq \sqrt{2}, \qquad \text{for $i,j = 1,2,\ldots,s$}. \end{equation}
\end{lemma}
\begin{proof}
	Consider the discrete cosine transform (DCT) matrix $\Theta \in \mathbb{R}^{s \times s}$, given by
	\[ \Theta_{i,1} = \frac{1}{\sqrt{s}}, \ i = 1,2,\ldots,s, \qquad \Theta_{i,j} = \sqrt{\frac{2}{s}} \cos \frac{(2i - 1)(j-1)\pi}{2s}, \ i = 1,2,\ldots,s, \ j = 2,3,\ldots,s. \]
	$\Theta$ is orthogonal and satisfies \eqref{l_infty}.
\end{proof}
Let $\Theta \in \mathbb{R}^{s \times s}$ be the orthogonal matrix given in Lemma \ref{lemma:DCT}. We fix $\mathcal{K} \subseteq [d]$ to be the active feature set and denote by $\phi_{\mathcal{K}}$ the corresponding coordinates of $\phi$. We now construct $\phi_{\mathcal{K}}: \mathcal{X} \times \mathcal{A} \rightarrow \mathbb{R}^s$ as follows: For $i = 1,2,\ldots, \frac{s}{2}$, $k = \pm 1, \pm 2, \ldots, \pm(d-s)$, let \vspace{-.8em}
\[ \begin{array}{cccccccccc} & & & & \text{\footnotesize $2i-1$} & \text{\footnotesize $2i$} & & & & \\ \phi_{\mathcal{K}}(\overline{x},a_i) := \sqrt{\frac{s}{2}} \cdot \Theta \big( \!\!\! & 0 & 0 & \cdots & 1 & 0 & \cdots & 0 & 0 & \!\!\! \big)^{\top} \in \mathbb{R}^s, \end{array} \]
\[ \begin{array}{cccccccccc} & & & & \text{\footnotesize $2i-1$} & \text{\footnotesize $2i$} & & & & \\ \phi_{\mathcal{K}}(\overline{x},\bar{a}_{i,k}) := \sqrt{\frac{s}{2}} \cdot \Theta \big( \!\!\! & 0 & 0 & \cdots & 1 - \varsigma_1 &  \varsigma_1 & \cdots & 0 & 0 & \!\!\! \big)^{\top} \in \mathbb{R}^s, \end{array} \]
\[ \begin{array}{cccccccccc} & & & & \text{\footnotesize $2i-1$} & \text{\footnotesize $2i$} & & & & \\ \phi_{\mathcal{K}}(\underline{x},a_i) = \phi_{\mathcal{K}}(\underline{x},\bar{a}_{i,k}) := \sqrt{\frac{s}{2}} \cdot \Theta \big( \!\!\! & 0 & 0 & \cdots &  \varsigma_2 & 1 -  \varsigma_2 & \cdots & 0 & 0 & \!\!\! \big)^{\top} \in \mathbb{R}^s, \end{array} \]
where $\varsigma_1, \varsigma_2 \in (0,1)$ will be determined later. By construction, we have $\| \phi_{\mathcal{K}}(x,a) \|_{\infty} \leq 1$ for any $(x,a) \in \mathcal{X} \times \mathcal{A}$. Note that $\phi_{\mathcal{K}}$ abstracts all the dynamic informatrion for state-action pairs, and $\phi_{\mathcal{K}^c}$ does not affect the transition model or reward function. Therefore, it is sufficient for us to use $\phi_{\mathcal{K}}$ when identifying the optimal policy or calculate value functions.


We propose $\frac{s}{2}$ MDP models $M_1, M_2, \ldots, M_{\frac{s}{2}}$, where $M_i$ has transition kernel $p_i(x' \, | \, x,a) = \phi_{\mathcal{K}}(x,a)^{\top} \psi_i(x')$ given by
\[ \begin{array}{cccccccccc} & & & & \text{\footnotesize $2i-1$} & \text{\footnotesize $2i$} & & & & \\ \psi_i(\overline{x}) = \sqrt{\frac{2}{s}} \cdot \Theta \big( \!\!\! & 1 - \delta_1 & \delta_2 & \cdots & 1 & 0 & \cdots & 1 - \delta_1 & \delta_2 & \!\!\! \big)^{\top} \in \mathbb{R}^{s}, \end{array} \]
\[ \begin{array}{cccccccccc} & & & & \text{\footnotesize $2i-1$} & \text{\footnotesize $2i$} & & & & \\ \psi_i(\underline{x}) = \sqrt{\frac{2}{s}} \cdot \Theta \big( \!\!\! & \delta_1 & 1-\delta_2 & \cdots & 0 & 1 & \cdots & \delta_1 & 1-\delta_2 & \!\!\! \big)^{\top} \in \mathbb{R}^{s}. \end{array} \]
Here, $\delta_1, \delta_2 \in \big[0, 2(1-\gamma)\big)$ are parameters reflecting the small differences among actions.

The reward functions are the same for all models and are chosen as \[ r(\overline{x},a_i) = r(\overline{x},\bar{a}_{i,k}) = 1, \qquad r(\underline{x},a_i) = r(\underline{x},\bar{a}_{i,k}) = 0,\] for $i = 1,2,\ldots, \frac{s}{2}$, $j = \pm 1, \pm 2, \ldots, \pm (d-s)$.

\subsection{Analyzing the concentration of the likelihood ratio}

We devise a behavior policy $\bar{\pi}$ and verify the likelihood ratio condition under the data collecting scheme in Assumption \ref{assm:data_colle}.
We start from an initial distribution $\bar{\xi}_0$ and take a behavior policy $\bar{\pi}(\bar{a}_{i,k} \, | \, \overline{x}) = \bar{\pi}(\bar{a}_{i,k} \, | \, \underline{x}) = \frac{1}{s(d-s)}$ for any $i,k$. Under this specific $\bar{\pi}$, due to symmetry, all MDP models ${M}_1, {M}_2, \ldots, {M}_{\frac{s}{2}}$ have the same marginal distribution at each time step $l = 0,1,\ldots,L-1$, which we denote by $\bar{\xi}_l = \left( \!\! \begin{array}{c} \bar{\xi}_l(\overline{x}) \\ \bar{\xi}_l(\underline{x}) \end{array} \!\!\right) \in \mathbb{R}^2$. Define the average distribution as $ \bar{\xi} := \frac{1}{L} \sum_{l=0}^{L-1} \bar{\xi}_l \in \mathbb{R}^2$.

Take $$\begin{aligned} p_{\min} := \min \Big\{ & p_i(\overline{x} \, | \, \overline{x}, \bar{a}_{i,k}), p_i(\underline{x} \, | \, \overline{x}, \bar{a}_{i,k}), p_i(\overline{x} \, | \, \underline{x}, \bar{a}_{i,k})), p_i(\overline{x} \, | \, \overline{x}, \bar{a}_{i,k})), \\ & \qquad \qquad \qquad \qquad \qquad i = 1,2,\ldots,\frac{s}{2}, \, k = \pm 1, \pm 2 \ldots, \pm (d-s) \Big\} \end{aligned} $$
and
\[ \Sigma^{\circ} := \bar{\xi}(\overline{x}) \left(\!\! \begin{array}{c} 1 - \varsigma_1 \\ \varsigma_1 \end{array} \!\!\right) \big(\!\! \begin{array}{cc} 1 - \varsigma_1 & \varsigma_1 \end{array} \!\!\big) + \bar{\xi}(\underline{x}) \left(\!\! \begin{array}{c} \varsigma_2 \\ 1 - \varsigma_2 \end{array} \!\! \right) \big(\!\! \begin{array}{cc} \varsigma_2 & 1 - \varsigma_2 \end{array} \!\! \big). \]
Parallel to Lemma C.3 in \cite{duan2020minimax}, we provide concentration results of the likelihood ratio in Lemma \ref{lemma:like}. The proof can be found in Appendix \ref{appendix:proof:lemma:like}.
\begin{lemma} \label{lemma:like}
	If we take $\delta_1, \delta_2 \geq 0$ such that
	\begin{equation} \label{<} \big( \!\! \begin{array}{cc} \delta_1 & - \delta_2 \end{array} \!\! \big) \Sigma^{\circ} \left( \!\! \begin{array}{c} \delta_1 \\ - \delta_2 \end{array} \!\! \right) \leq \frac{s p_{\min}}{100N}, \qquad \delta_1 \vee \delta_2 \leq \frac{p_{\min}}{100\sqrt{L}}, \end{equation}
	then for any $i,j = 1,2,\ldots,s$, $i \neq j$, it holds that
	\begin{equation} \label{like} \mathbb{P}_i \bigg( \frac{\mathcal{L}_j(\mathcal{D})}{\mathcal{L}_i(\mathcal{D})}  \geq \frac{1}{2} \bigg) \geq \frac{1}{2}. \end{equation}
\end{lemma}
Lemma \ref{lemma:like} suggests that as long as \eqref{<} is satisfied, the likelihood test in Lemma \ref{lemma:LikeTest} works for any pair of indices $(\alpha, \beta) = (i,j)$, $i \neq j$.

\subsection{Calculating the gap in values}

For model $M_i$, the optimal policy is given by \[ \pi_i^{*}(\overline{x}) = a_i \qquad \text{and} \qquad \pi_i^{*}(\underline{x}) = \left\{ \begin{aligned} & \text{$a_j$ for any $j \neq i$}, & & \text{if $(1-\varsigma_2)\delta_2 > \varsigma_2\delta_1$}, \\ & \text{$a_i$}, & & \text{otherwise}. \end{aligned} \right. \]
For computational simplicity, we take initial distribution $\xi_0 := \left( \!\! \begin{array}{c} \xi_0(\overline{x}) \\ \xi_0(\underline{x}) \end{array} \!\! \right) = \left( \!\! \begin{array}{c} 1 \\ 0 \end{array} \!\! \right) \in \mathbb{R}^2$. In the following Lemma \ref{lemma:v-v}, we provide an estimation for the difference between values of optimal and sub-optimal policies. See Appendix \ref{appendix:proof:lemma:v-v} for the proof.

\begin{lemma} \label{lemma:v-v}
	If $\delta_1 \leq \frac{1-\gamma}{\gamma}$, $\delta_2 \leq \varsigma_2$, then for any policy $\pi$ such that $\pi(\overline{x}) \neq a_i$, it holds that
	\[ v_{M_i, \xi_0}^{*} - v_{M_i, \xi_0}^{\pi} \geq \frac{\gamma\delta_1}{2(1-\gamma)} \cdot \frac{1}{1 - \gamma + 2\gamma\varsigma_2}. \]
\end{lemma}
According to Lemma \ref{lemma:v-v}, if we take
\begin{equation} \label{v-v_0}  \rho_i = \rho' := \frac{\gamma\delta_1}{2(1-\gamma)} \cdot \frac{1}{1 - \gamma + 2\gamma\varsigma_2}, \qquad i=1,2,\ldots,\frac{s}{2}, \end{equation}
then condition \eqref{HardInstance} in Lemma \ref{lemma:LikeTest} holds for any $(\alpha, \beta) = (i,j)$, $i \neq j$.

\subsection{Choosing parameters}

We now integrate Lemmas \ref{lemma:LikeTest}, \ref{lemma:like} and \ref{lemma:v-v}. Specifically, we choose parameters $\varsigma_1, \varsigma_2, \delta_1, \delta_2$ and $\bar{\xi}$ that maximize $\rho'$ in \eqref{v-v_0} under the constraint \eqref{<}.

	We first consider the optimization problem
	\[ {\rm maximize} \quad \delta_1, \qquad \text{subject to} \quad \big( \!\! \begin{array}{cc} \delta_1 & - \delta_2 \end{array} \!\! \big) \Sigma^{\circ} \left( \!\! \begin{array}{c} \delta_1 \\ - \delta_2 \end{array} \!\! \right) \leq \frac{s p_{\min}}{100N}. \]
	It has solution
	\begin{equation} \label{delta12} \delta_1 = \sqrt{\frac{\Sigma_{22}^{\circ}}{{\rm det}(\Sigma^{\circ})}} \sqrt{\frac{sp_{\min}}{100N}}, \qquad \delta_2 = \sqrt{\frac{(\Sigma_{12}^{\circ})^2}{\Sigma_{22}^{\circ}{\rm det}(\Sigma^{\circ})}} \sqrt{\frac{sp_{\min}}{100N}}. \end{equation}
	Plugging \eqref{delta12} into \eqref{v-v_0} and assuming that $p_{\min} \geq \frac{\varsigma_2}{2}$, we have
	\begin{equation} \label{rho_1} \rho' \geq \frac{\gamma}{2(1-\gamma)} \cdot \frac{\sqrt{\varsigma_2}}{1 - \gamma + 2\gamma\varsigma_2} \cdot \sqrt{\frac{\Sigma_{22}^{\circ}}{{\rm det}(\Sigma^{\circ})}} \sqrt{\frac{s}{200N}}. \end{equation}
	We maximize the right hand side of \eqref{rho_1} over $\varsigma_2$, and obtain
	\[ \varsigma_2 = \frac{1-\gamma}{2\gamma}, \qquad \rho' \geq \frac{\sqrt{\gamma}}{80} \sqrt{\frac{\Sigma_{22}^{\circ}}{{\rm det}(\Sigma^{\circ})}} \sqrt{\frac{s}{N(1-\gamma)^3}}. \]
	We further let $\varsigma_1 \in \big[ \frac{1-\gamma}{2\gamma}, 1 - \frac{1-\gamma}{2\gamma} \big)$ and suppose the sample size
	\begin{equation} \label{cond2} N \geq \frac{(\Sigma_{22}^{\circ} \vee \Sigma_{12}^{\circ})^2}{\Sigma_{22}^{\circ}{\rm det}(\Sigma^{\circ})} \frac{400 sL}{1-\gamma}. \end{equation}
	In this case, $p_{\min} \geq \frac{\varsigma_2}{2}$ and $\delta_1 \vee \delta_2 \leq \frac{p_{\min}}{100\sqrt{L}} \leq \varsigma_2 \leq \frac{1-\gamma}{\gamma}$.
	
	In summary, if the sample size $N$ satisfies \eqref{cond2} and we take $$\gamma \geq \frac{1}{2}, \qquad \varsigma_1 \in \Big[ \frac{1-\gamma}{2\gamma}, 1 - \frac{1-\gamma}{2\gamma} \Big), \quad \varsigma_2 = \frac{1-\gamma}{2\gamma} \qquad \text{and} \qquad \text{$\delta_1, \delta_2$ in \eqref{delta12}}, $$ then the conditions in Lemmas \ref{lemma:like} and \ref{lemma:v-v} are satisfied and \eqref{LikeTest} holds for 
	\begin{equation} \label{rho} 
	\rho := \frac{1}{80\sqrt{2}} \sqrt{\frac{\Sigma_{22}^{\circ}}{{\rm det}(\Sigma^{\circ})}} \sqrt{\frac{s}{N(1-\gamma)^3}}. \end{equation}
	
	Remark that under this construction, we still have the flexibility to take $\varsigma_1 \nearrow 1 - \frac{1-\gamma}{2\gamma}$ so that $\Sigma^{\circ}$ is very ill-conditioned. For instance, if we take $\varsigma_1 = 1 - \frac{1-\gamma}{\gamma}$, then ${\rm det}(\Sigma^{\circ})$ or $\lambda_{\min}(\Sigma^{\circ})$ at least has the order of $(1-\gamma)^{3}$.
	
	In order that condition \eqref{cond2} is as weak as possible,
	we take $\gamma \geq \frac{2}{3}$, $\varsigma_1 = \frac{1 - \gamma}{2\gamma}$ and $\bar{\xi}(\overline{x}) = \bar{\xi}(\underline{x}) = \frac{1}{2}$. In this setting, if $N \geq 2000 sL(1-\gamma)^{-1}$ then \eqref{cond2} holds.

	\subsection{Relating to mismatch terms}
	
	In this part, we relate $\frac{\Sigma_{22}^{\circ}}{{\rm det}(\Sigma^{\circ})}$ in \eqref{rho} to mismatch terms $\chi_{\mathcal{G}(\mathcal{K})}^2(\mu^{*},\bar{\mu})$ and $C_{\min}(\Sigma,s)$.
	
	\subsubsection{Restricted $\chi^2$-divergence}
	
	According to Lemma B.2 in \cite{duan2020minimax}, we have
	\[ 1 + \chi_{\mathcal{G}(\mathcal{K})}^2(\mu^{*}, \bar{\mu}) = \sup_{f \in \mathcal{G}(\mathcal{K})} \frac{\mathbb{E}_i\big[ f(x, a) \, \big| \, (x,a) \sim \mu^{*} \big]^2}{\mathbb{E}_i\big[ f^2(x,a) \, \big| \, (x,a) \sim \bar{\mu} \big]} = (\nu_{\mathcal{K}}^{*})^{\top} \Sigma_{\mathcal{K}}^{-1} \nu_{\mathcal{K}}^{*}, \]
	where
	\[ \nu_{\mathcal{K}}^{*} := \mathbb{E}_i\big[ \phi_{\mathcal{K}}(x,a) \, \big| \, (x,a) \sim \mu^{*} \big] = \frac{1}{1-\gamma} \sum_{t=0}^{\infty} \gamma^t \mathbb{E}_i\big[ \phi_{\mathcal{K}}(x_t,a_t) \, \big| \, x_0 \sim \xi_0, \pi_i^{*} \big] \in \mathbb{R}^s\]
	and
	\[ \begin{aligned} \Sigma_{\mathcal{K}} := & \mathbb{E}_i \big[ \phi_{\mathcal{K}}(x,a) \phi_{\mathcal{K}}(x,a)^{\top} \, \big| \, (x,a) \sim \bar{\mu} \big] \\ = & \mathbb{E}_{i}\Bigg[ \frac{1}{L}\sum_{h=0}^{L-1} \phi_{\mathcal{K}}(x_h^{(k)}, a_h^{(k)}) \phi_{\mathcal{K}}(x_h^{(k)}, a_h^{(k)})^{\top} \, \Bigg| \, x_0^{(k)} \sim \bar{\xi}_0, \bar{\pi} \Bigg] \in \mathbb{R}^{s \times s}. \end{aligned} \]
	
	For model $M_i$, $\overline{x}$ is an absorbing state under the optimal policy $\pi_i^{*}$. Therefore, $\mu^{*}(\overline{x}) = 1$ and $\nu_{\mathcal{K}}^{*} =  \phi_{\mathcal{K}}(\overline{x}, a_i)$.
	Under our proposed behavior policy $\bar{\pi}$, we have
	\begin{equation} \label{SigmaK} \Sigma_{\mathcal{K}} = \Theta \left( \!\! \begin{array}{cccc} \Sigma^{\circ} & 0 & \cdots & 0 \\ 0 & \Sigma^{\circ} & \cdots & 0 \\ \vdots & \vdots & \ddots & \vdots \\ 0 & 0 & \cdots & \Sigma^{\circ} \end{array} \!\! \right) \Theta^{\top}. \end{equation}
	It follows that
	\[ 1 + \chi_{\mathcal{G}(\mathcal{K})}^2(\mu^{*}, \bar{\mu}) = \big( \nu_{\mathcal{K}}^{*} \big)^{\top} \Sigma_{\mathcal{K}}^{-1} \nu_{\mathcal{K}}^{*} = \frac{s}{2} \big((\Sigma^{\circ})^{-1}\big)_{1,1} = \frac{s\Sigma_{22}^{\circ}}{2 {\rm det}(\Sigma^{\circ})}. \]
	To this end, we have
	\[ \rho = \frac{1}{80} \sqrt{1 + \chi_{\mathcal{G}(\mathcal{K})}^2(\mu^{*}, \bar{\mu})} \frac{1}{\sqrt{N(1-\gamma)^3}}. \]
	This implies there always exists an DMDP instance $M \in \mathcal{M}_{\phi,s}(\mathcal{X}, \mathcal{A}, \gamma)$ with feature $\phi \in (\mathbb{R}^d)^{\mathcal{X} \times \mathcal{A}}$ satisfying $\|\phi(x,a)\|_{\infty} \leq 1$ for all $(x,a) \in \mathcal{X} \times \mathcal{A}$,
	then 
	\[ \mathbb{P}_{M} \Bigg( v_{\xi_0}^{*} - v_{\xi_0}^{\widehat{\pi}(\mathcal{D})} \gtrsim \frac{1}{(1-\gamma)^{\frac{3}{2}}} \sqrt{1+\chi_{\mathcal{G}(\mathcal{K})}^2(\mu^{*},\bar{\mu})} \sqrt{\frac{1}{N}} \Bigg) \geq \frac{1}{6}, \]
	where $\mu^{*}$ is the discounted state-action occupancy measure of $\pi^*$.


	\subsubsection{Restricted minimum eigenvalue}\label{sec:worse_PO_lower}

	The uncentered covariance matrix $\Sigma \in \mathbb{R}^{d \times d}$ is given by
	\[ \begin{aligned} 
	\Sigma := & \mathbb{E}_{i}\Bigg[ \frac{1}{L}\sum_{h=0}^{L-1} \phi(x_h^{(k)}, a_h^{(k)}) \phi(x_h^{(k)}, a_h^{(k)})^{\top} \, \Bigg| \, x_0^{(k)} \sim \bar{\xi}_0, \bar{\pi} \Bigg] \in \mathbb{R}^{s \times s}.
	\end{aligned} \]
	In the following, we specify the choice of $\phi_{\mathcal{K}^c}(\overline{x}, \bar{a}_{i,k})$ and $\phi_{\mathcal{K}^c}(\underline{x}, \bar{a}_{i,k})$ and show that if
	\begin{equation} \label{cond1} \bar{\xi}(\overline{x}) \varsigma_1^2 + \bar{\xi}(\underline{x}) (1 - \varsigma_2)^2 \geq \bar{\xi}(\overline{x}) (1-\varsigma_1)^2 + \bar{\xi}(\underline{x}) \varsigma_2^2, \end{equation}
	then
	\begin{equation} \label{lb_Cmin} \frac{\Sigma_{22}^{\circ}}{{\rm det}(\Sigma^{\circ})} \geq \frac{1}{2 C_{\min}(\Sigma, s) }. \end{equation}
	
	Under condition \eqref{cond1}, it holds that $\Sigma_{22}^{\circ} \geq \Sigma_{11}^{\circ}$, therefore, $\frac{\Sigma_{22}^{\circ}}{{\rm det}(\Sigma^{\circ})} \geq \frac{{\rm Tr}(\Sigma^{\circ})}{2 {\rm det}(\Sigma^{\circ})}$.
	In addition, for the $2$-by-$2$ matrix $\Sigma^{\circ}$, we have $\lambda_{\min}(\Sigma^{\circ}) + \lambda_{\max}(\Sigma^{\circ}) = {\rm Tr}(\Sigma^{\circ})$ and $\lambda_{\min}(\Sigma^{\circ}) \lambda_{\max}(\Sigma^{\circ}) = {\rm det}(\Sigma^{\circ})$. It follows that 
	$$ \frac{\Sigma_{22}^{\circ}}{{\rm det}(\Sigma^{\circ})} \geq \frac{{\rm Tr}(\Sigma^{\circ})}{2{\rm det}(\Sigma^{\circ})} \geq \frac{\lambda_{\max}(\Sigma^{\circ})}{2\lambda_{\max}(\Sigma^{\circ})\lambda_{\min}(\Sigma^{\circ})} = \frac{1}{2 \lambda_{\min}(\Sigma^{\circ})}.$$
	We next relate $\lambda_{\min}(\Sigma^{\circ})$ to $C_{\min}(\Sigma, s)$.
	
	Let $\bar{\Theta} \in \mathbb{R}^{(d-s) \times (d-s)}$ be an orthogonal matrix given by Lemma \ref{lemma:DCT}. We take
	\[ \phi_{\mathcal{K}^c}(\overline{x}, \bar{a}_{i,k}) = \phi_{\mathcal{K}^c}(\underline{x}, \bar{a}_{i,k}) := {\rm sign}(k) \sqrt{\frac{d-s}{2}} \cdot {\rm col}_k(\bar{\Theta}), \] for $i = 1,2,\ldots, \frac{s}{2}$, $k = \pm 1, \pm 2, \ldots, \pm (d-s)$.
	It holds that $\|\phi_{\mathcal{K}^c}(\overline{x}, \bar{a}_{i,k})\|_{\infty} \leq 1$ and $\|\phi_{\mathcal{K}^c}(\underline{x}, \bar{a}_{i,k})\|_{\infty} \leq 1$.
	For notational simplicity, let $\mathcal{K} = [s]$.
	Under our proposed behavior policy $\bar{\pi}(\bar{a}_{i,k} \, | \, \overline{x}) = \bar{\pi}(\bar{a}_{i,k} \, | \, \underline{x}) = \frac{1}{s(d-s)}$, we have
	\begin{equation} \label{Sigma} \Sigma = \left( \begin{array}{cc} \Sigma_{\mathcal{K}} & 0 \\ 0 & \frac{1}{2} I_{d-s} \end{array} \right). \end{equation}
	
	By \eqref{SigmaK}, $\lambda_{\min}(\Sigma_{\mathcal{K}}) = \lambda_{\min}(\Sigma^{\circ})$. We also note that ${\rm Tr}(\Sigma^{\circ}) = \bar{\xi}(\overline{x}) \| (1-\varsigma_1, \varsigma_1) \|_2^2 + \bar{\xi}(\underline{x}) \| (\varsigma_2, 1-\varsigma_2) \|_2^2 \leq 1$, and therefore
	\[ \lambda_{\min}(\Sigma^{\circ}) \leq \frac{{\rm Tr}(\Sigma^{\circ})}{2} \leq \frac{1}{2}. \]
	It follows that $\lambda_{\min}(\Sigma) = \lambda_{\min}(\Sigma^{\circ})$, which further implies	$C_{\min}(\Sigma, s) \geq \lambda_{\min}(\Sigma) = \lambda_{\min}(\Sigma^{\circ})$.
	On the other hand, the eigenvector of $\Sigma$ corresponding to $\lambda_{\min}(\Sigma^{\circ})$ has support set $\mathcal{K}$ and is $s$-sparse. Hence, $\lambda_{\min}(\Sigma^{\circ}) \geq C_{\min}(\Sigma,s)$. In this way, we have proved $C_{\min}(\Sigma, s) = \lambda_{\min}(\Sigma^{\circ})$ for $\Sigma$ defined in \eqref{Sigma}.
	
	In the special case where $\varsigma_1 = \varsigma_2 = \frac{1-\gamma}{2\gamma}$ and $\bar{\xi}(\overline{x}) = \bar{\xi}(\underline{x}) = \frac{1}{2}$, condition \eqref{cond1} holds. Plugging \eqref{lb_Cmin} into \eqref{rho}, we finish our proof of Theorem \ref{thm:lowerbound}.

%% file: lower_bound_supp.tex
\subsubsection{Proof of Lemma \ref{lemma:like}} \label{appendix:proof:lemma:like}

\begin{proof}[Proof of Lemma \ref{lemma:like}]
	It is easy to see that
	\[ \ln \frac{\mathcal{L}_j(\mathcal{D})}{\mathcal{L}_i(\mathcal{D})} = \sum_{k=1}^K \sum_{l=0}^{L-1} \ln \frac{p_j(x_{l+1}^{(k)} \, | \, x_l^{(k)}, a_l^{(k)})}{p_i(x_{l+1}^{(k)} \, | \, x_l^{(k)}, a_l^{(k)})} = \sum_{k=1}^K \sum_{l=0}^{L-1} \ln (1 - \Lambda_{l}^{(k)}), \]
	where 
	\[ \Lambda_l^{(k)} := \frac{p_i(x_{l+1}^{(k)} \, | \, x_l^{(k)}, a_l^{(k)}) - p_j(s_{l+1}^{(k)} \, | \, s_l^{(k)}, a_l^{(k)})}{p_i(s_{l+1}^{(k)} \, | \, s_l^{(k)}, a_l^{(k)})} = \frac{\phi(s_l^{(k)}, a_l^{(k)})^{\top} \big( \psi_i(s_{l+1}^{(k)}) - \psi_j(s_{l+1}^{(k)})\big)}{p_i(s_{l+1}^{(k)} \, | \, s_l^{(k)}, a_l^{(k)})}. \]
	If we take $\delta_1 \vee \delta_2 \leq \frac{p_{\min}}{2}$, then $|\Lambda_l^{(k)}| \leq \frac{1}{2}$ and
	\begin{equation} \label{E1+E2} \ln \frac{\mathcal{L}_j(\mathcal{D})}{\mathcal{L}_i(\mathcal{D})} \geq - \underbrace{\sum_{k=1}^K \sum_{l=0}^{L-1} \Lambda_l^{(k)}}_{E_1} - \underbrace{\sum_{k=1}^K \sum_{l=0}^{L-1} \big(\Lambda_l^{(k)}\big)^2}_{E_2}. \end{equation}
	
	Since $\mathbb{E}_i [\Lambda_l^{(k)} \, | \, s_l^{(k)}, a_l^{(k)}] = 0$, we apply Freedman's inequality to analyze $E_1$.
	The conditional variances satisfy
	\[ \begin{aligned} & \mathbb{E}_i \big[ \big(\Lambda_l^{(k)}\big)^2 \, \big| \, s_l^{(k)}, a_l^{(k)} \big] \\ = & p_i( \overline{x} \, | \, s_l^{(k)}, a_l^{(k)}) \bigg( \frac{\phi(s_l^{(k)}, a_l^{(k)})^{\top} \big( \psi_i(\overline{x}) - \psi_j(\overline{x}) \big) }{p_i(\overline{x} \, | \, s_l^{(k)}, a_l^{(k)})} \bigg)^2 \\ & + p_i( \underline{x} \, | \, s_l^{(k)}, a_l^{(k)}) \bigg( \frac{\phi(s_l^{(k)}, a_l^{(k)})^{\top} \big( \psi_i(\underline{x}) - \psi_j(\underline{x}) \big) }{p_i(\underline{x} \, | \, s_l^{(k)}, a_l^{(k)})} \bigg)^2 \\ = & \frac{\Big( \phi(s_l^{(k)}, a_l^{(k)})^{\top} \big( \psi_i(\overline{x}) - \psi_j(\overline{x}) \big) \Big)^2 }{p_i(\overline{x} \, | \, s_l^{(k)}, a_l^{(k)})p_i(\underline{x} \, | \, s_l^{(k)}, a_l^{(k)})} \leq \frac{1}{p_{\min}(1-p_{\min})} \Big( \phi(s_l^{(k)}, a_l^{(k)})^{\top} \big( \psi_i(\overline{x}) - \psi_j(\overline{x}) \big) \Big)^2. \end{aligned} \]
	Denote $\Xi_{k} := \frac{1}{L} \sum_{l=0}^{L-1} \big( \phi(s_l^{(k)}, a_l^{(k)})^{\top} ( \psi_i(\overline{x}) - \psi_j(\overline{x}) ) \big)^2$.
	Note that
	\[ \begin{array}{ccccccccccccc} & & & & \!\!\!\text{\footnotesize $2i-1$}\!\!\! & \!\!\!\text{\footnotesize $2i$}\!\!\! & & \!\!\!\text{\footnotesize $2j-1$}\!\!\! & \!\!\!\text{\footnotesize $2j$}\!\!\! & & & & \\ \psi_i(\overline{x}) - \psi_j(\overline{x}) = \sqrt{\frac{2}{s}} \cdot \Theta \big( \!\!\! & 0 & 0 & \cdots & \delta_1 & - \delta_2 & \cdots & - \delta_1 & \delta_2 & \cdots & 0 & 0 & \!\!\! \big)^{\top} \in \mathbb{R}^{s}, \end{array} \]
	therefore,
	\[ \mathbb{E} [ \Xi_k ] = \frac{4}{s} \big( \!\! \begin{array}{cc} \delta_1 & - \delta_2 \end{array} \!\! \big) \Sigma^{\circ} \left( \!\! \begin{array}{c} \delta_1 \\ - \delta_2 \end{array} \!\! \right) \qquad
	\text{and} \qquad
	|\Xi_k| \leq (\delta_1 \vee \delta_2)^2. \]
	By Bernstein inequality and the independence of trajectories ${\tau}_1, {\tau}_2, \ldots, {\tau}_K$, we have with $\mathbb{P}_i$-probability at least $\frac{5}{6}$,
	\begin{equation} \label{CondVar} 
	\frac{1}{K} \sum_{k=1}^K \Xi_k \leq \Bigg( \sqrt{\mathbb{E}[\Xi_k]} + (\delta_1 \vee \delta_2) \sqrt{\frac{2 \ln 6}{3 K}} \Bigg)^2 =: \sigma^2. \end{equation}
	Since $|\Lambda_l^{(k)}| \leq p_{\min}^{-1}(\delta_1 \vee \delta_2)$, by Freedman's inequality, with $\mathbb{P}_i$-probability at least $\frac{5}{6}$,
	\begin{equation} \label{Freedman} \frac{1}{N} \sum_{k=1}^K \sum_{l=0}^{L-1} \Lambda_l^{(k)} \leq \frac{\sigma}{\sqrt{p_{\min}(1-p_{\min})}} \sqrt{\frac{2 \ln 6}{N}} + p_{\min}^{-1}(\delta_1 \vee \delta_2)\frac{2 \ln 6}{3 N} \qquad \text{and} \qquad \text{\eqref{CondVar} holds}. \end{equation}
	Combining \eqref{CondVar} and \eqref{Freedman}, we use union bound and derive that with $\mathbb{P}_i$-probabliity at least $\frac{2}{3}$,
	\begin{equation} \label{E1} \sum_{k=1}^K \sum_{l=0}^{L-1} \Lambda_l^{(k)} \leq 4\sqrt{\ln 6} \sqrt{\big( \!\! \begin{array}{cc} \delta_1 & - \delta_2 \end{array} \!\! \big) \Sigma^{\circ} \left( \!\! \begin{array}{c} \delta_1 \\ - \delta_2 \end{array} \!\! \right)}\sqrt{\frac{N}{p_{\min} s}} +  \frac{2 \ln 6}{3} (\delta_1 \vee \delta_2) \bigg( \sqrt{\frac{6L}{p_{\min}}} + \frac{1}{p_{\min}} \bigg). \end{equation}
	
	As for $E_2$, since
	\[ \mathbb{E}_i\Bigg[ \frac{1}{L} \sum_{l=0}^{L-1} \big(\Lambda_l^{(k)}\big)^2 \Bigg] \leq \frac{4}{sp_{\min}(1-p_{\min})} \big( \!\! \begin{array}{cc} \delta_1 & - \delta_2 \end{array} \!\! \big) \Sigma^{\circ} \left( \!\! \begin{array}{c} \delta_1 \\ - \delta_2 \end{array} \!\! \right), ~~ \Bigg| \frac{1}{L} \sum_{l=0}^{L-1} \big(\Lambda_l^{(k)}\big)^2 \Bigg| \leq \frac{(\delta_1 \vee \delta_2)^2}{p_{\min}^2}, \]
	by Bernstein's inequality, with $\mathbb{P}_i$-probability at least $\frac{5}{6}$,
	\begin{equation} \label{E2} \sum_{k=1}^K \sum_{l=0}^{L-1} \big(\Lambda_l^{(k)}\big)^2 \leq \Bigg( 2\sqrt{2} \sqrt{ \big( \!\! \begin{array}{cc} \delta_1 & - \delta_2 \end{array} \!\! \big) \Sigma^{\circ} \left( \!\! \begin{array}{c} \delta_1 \\ - \delta_2 \end{array} \!\! \right)} \sqrt{\frac{N}{s p_{\min}}} + \frac{\delta_1 \vee \delta_2}{p_{\min}} \sqrt{\frac{2 L \ln 6}{3}} \Bigg)^2. \end{equation}
	
	Plugging \eqref{E1} and \eqref{E2} into \eqref{E1+E2} and applying condition \eqref{<}, we obtain \eqref{like}.
	
\end{proof}

\subsubsection{Proof of Lemma \ref{lemma:v-v}} \label{appendix:proof:lemma:v-v}

\begin{proof}[Proof of Lemma \ref{lemma:v-v}]
	
	We consider another policy $\pi_i'$ such that $\pi_i'(\overline{x}) = a_j$ for some $a_j \neq a_i$ and $\pi_i'(\underline{x}) = \pi_i^{*}(\underline{x})$. It holds that
	\[ v_{M_i, \xi_0}^{\pi_i'} = \sup \big\{ v_{M_i, \xi_0}^{\pi} \, \big| \, \pi(\overline{x}) \neq a_i \big\}. \]
	For any $\pi$, denote its corresponding transition matrix by $P_i^{\pi} = \left( \!\! \begin{array}{cc} p_i^{\pi}(\overline{x} \, | \, \overline{x}) & p_i^{\pi}(\underline{x} \, | \, \overline{x}) \\ p_i^{\pi}(\overline{x} \, | \, \overline{x}) & p_i^{\pi}(\underline{x} \, | \, \underline{x}) \end{array} \!\! \right) \in \mathbb{R}^{2 \times 2}$. 
	Define $v_{M_i}^{*} := \left( \!\! \begin{array}{c} v_{M_i}^{*}(\overline{x}) \\ v_{M_i}^{*}(\underline{x}) \end{array} \!\! \right) \in \mathbb{R}^2$. We have the following decomposition,
	\begin{equation} \label{v-v} v_{M_i, \xi_0}^{*} - v_{M_i, \xi_0}^{\pi_i'} = \sum_{t=0}^{\infty} \gamma^{t+1} \xi_0^{\top} \big( P_i^{\pi_i'} \big)^t \big(P_i^{\pi_i^{*}} - P_i^{\pi_i'}\big) v_{M_i}^{*} = \gamma \xi_0^{\top} \big( I - \gamma P_i^{\pi_i'} \big)^{-1} \big(P_i^{\pi_i^{*}} - P_i^{\pi_i'}\big) v_{M_i}^{*}. \end{equation}
	
	Under model $M_i$, when $\delta_2 \leq \varsigma_2$, $\pi_i^{*}$ and $\pi_i'$ satisfy
	\[ p_i^{\pi_i^{*}}(\overline{x} \, | \, \overline{x}) = 1, \quad p_i^{\pi_i'}(\overline{x} \, | \, \overline{x}) = 1 - \delta_1, \qquad p_i^{\pi_i^{*}}(\underline{x} \, | \, \overline{x}) = 0, \quad p_i^{\pi_i'}(\underline{x} \, | \, \overline{x}) = \delta_1, \] \[
	p_i^{\pi_i^{*}}(\overline{x} \, | \, \underline{x}) = p_i^{\pi_i'}(\overline{x} \, | \, \underline{x}) \leq 2\varsigma_2, \qquad p_i^{\pi_i^{*}}(\underline{x} \, | \, \underline{x}) = p_i^{\pi_i'}(\underline{x} \, | \, \underline{x}) \geq 1 - 2\varsigma_2, \]
	therefore,
	\[ \big(P_i^{\pi_i^{*}} - P_i^{\pi_i'}\big) v_{M_i}^{*} = \left( \!\! \begin{array}{cc} \delta_1 & -\delta_1 \\ 0 & 0 \end{array} \!\! \right) v_{M_i}^{*} = \left( \!\! \begin{array}{c} \delta_1 \\ 0 \end{array} \!\! \right) \cdot \big( v_{M_i}^{*}(\overline{x}) - v_{M_i}^{*}(\underline{x}) \big). \]
	To this end, we reduce \eqref{v-v} into
	\begin{equation} \label{v-v_1} v_{M_i, \xi_0}^{*} - v_{M_i, \xi_0}^{\pi_i'} = \gamma \xi_0^{\top} \big( I - \gamma P_i^{\pi_i'} \big)^{-1} \left( \!\! \begin{array}{c} \delta_1 \\ 0 \end{array} \!\! \right) \cdot \big( v_{M_i}^{*}(\overline{x}) - v_{M_i}^{*}(\underline{x}) \big). \end{equation}
	
	Note that
	\[ v_{M_i}^{*}(\overline{x}) = (1-\gamma)^{-1} \qquad \text{and} \qquad v_{M_i}^{*}(\underline{x}) = \frac{\gamma}{1 - \gamma} \cdot \frac{p_i^{\pi_i^{*}}(\underline{x} \, | \, \underline{x})}{1 - \gamma\big(1 - p_i^{\pi_i^{*}}(\underline{x} \, | \, \underline{x})\big)}. \]
	Therefore,
	\begin{equation} \label{v-v_3} v_{M_i}^{*}(\overline{x}) - v_{M_i}^{*}(\underline{x}) = \frac{1}{1-\gamma} - \frac{\gamma}{1 - \gamma} \cdot \frac{p_i^{\pi_i^{*}}(\underline{x} \, | \, \underline{x})}{1 - \gamma\big(1 - p_i^{\pi_i^{*}}(\underline{x} \, | \, \underline{x})\big)} = \frac{1}{1 - \gamma p_i^{\pi_i^{*}}(\underline{x} \, | \, \underline{x})} \geq \frac{1}{1 - \gamma (1 - 2\varsigma_2)}. \end{equation}
	In addition, we have
	\[ \xi_0^{\top} \big( I - \gamma P_i^{\pi_i'} \big)^{-1} \left( \!\! \begin{array}{c} \delta_1 \\ 0 \end{array} \!\! \right) = \frac{\delta_1}{1-\gamma} \cdot \frac{1 - \gamma p_i^{\pi_i^{*}}(\underline{x} \, | \, \underline{x})}{1 + \gamma\delta_1 - \gamma p_i^{\pi_i^{*}}(\underline{x} \, | \, \underline{x})} \]
	Under the condition $\delta_1 \leq \frac{1-\gamma}{\gamma}$, we have $\gamma \delta_1 \leq 1 - \gamma \leq 1 - \gamma p_i^{\pi_i^{*}}(\underline{x} \, | \, \underline{x})$, therefore,
	\begin{equation} \label{v-v_4} \xi_0^{\top} \big( I - \gamma P_i^{\pi_i'} \big)^{-1} \left( \!\! \begin{array}{c} \delta_1 \\ 0 \end{array} \!\! \right) \geq \frac{\delta_1}{2(1-\gamma)}. \end{equation}
	Plugging \eqref{v-v_3} and \eqref{v-v_4} into \eqref{v-v_1}, we finish our proof.
\end{proof}